\documentclass{ieeetj}
\usepackage{cite}
\usepackage{amsmath,amssymb,amsfonts}
\usepackage{algorithmic}
\usepackage{graphicx,color}
\usepackage{textcomp}
\usepackage{xcolor}
\usepackage[dvipsnames]{xcolor}
\usepackage{hyperref} 
\hypersetup{hidelinks=true}
\usepackage{algorithm,algorithmic}
\def\BibTeX{{\rm B\kern-.05em{\sc i\kern-.025em b}\kern-.08em
    T\kern-.1667em\lower.7ex\hbox{E}\kern-.125emX}}
\AtBeginDocument{\definecolor{tmlcncolor}{cmyk}{0.93,0.59,0.15,0.02}\definecolor{NavyBlue}{RGB}{0,86,125}}

\usepackage{amssymb}
\usepackage{amsmath}
\usepackage{graphicx}
\usepackage{amsthm}
\usepackage{rotating}
\usepackage{adjustbox}
\usepackage{caption}
\usepackage{subcaption}
\usepackage{bm}
\usepackage{siunitx}
\usepackage{mathtools}
\usepackage{braket}
\usepackage{booktabs}
\usepackage{xcolor}
\usepackage{diagbox, xcolor}
\usepackage{float}

\newcommand{\ve}[1]{\mathbf{#1}}

\newcommand{\sigmu}[1]{\lceil {#1}, \mu\rfloor}

\DeclareMathOperator{\sign}{sgn}
\DeclareMathOperator{\sig}{sig}
\DeclareMathOperator{\proj}{Proj}
\DeclareMathOperator{\atan2}{atan2}
\DeclareMathOperator{\wrap}{wrap}

\newtheorem{assumption}{Assumption}
\newtheorem{theorem}{Theorem}[section]
\newtheorem{remark}{Remark}
\newtheorem{lemma}[theorem]{Lemma}

\def\OJlogo{\vspace{-4pt}$<$Society logo(s) and publication title will appear here.$>$}
\def\seclogo{\vspace{10pt}$<$Society logo(s) and publication title will appear here.$>$}

\def\authorrefmark#1{\ensuremath{^{\textbf{#1}}}}

\begin{document}
\receiveddate{XX Month, XXXX}
\reviseddate{XX Month, XXXX}
\accepteddate{XX Month, XXXX}
\publisheddate{XX Month, XXXX}
\currentdate{XX Month, XXXX}
\doiinfo{XXXX.2022.1234567}

\markboth{}{Author {et al.}}

\title{An Adaptive Fixed-Time Line-of-Sight Guidance Scheme for 3D Path Following of Underwater Vehicles: Theory and Experiment}

\author{Hanzhi Yang \authorrefmark{1}, Jalil Chavez-Galaviz \authorrefmark{1},\\ and Nina Mahmoudian \authorrefmark{1}}
\affil{School of Mechanical Engineering, Purdue University, West Lafayette, IN 47907 USA}
\corresp{Corresponding author: Nina Mahmoudian (email: ninam@purdue.edu).}
\authornote{This work was supported in part by Office of Naval Research (ONR) under Grant No. N00014-24-1-2019. The authors also wish to thank Chabely Pecina Martinez for her valuable assistance during the field validation. }

\begin{abstract}
{Reliable path tracking is crucial for autonomous underwater vehicles (AUVs) operating in dynamic and uncertain marine environments. However, traditional line-of-sight (LOS) guidance methods rely on asymptotic convergence, resulting in slow disturbance recovery and unpredictable tracking performance. Existing robust control methods typically require modifications to the underlying vehicle controller, limiting their practical application on commercial AUV platforms. This paper proposes a robust fixed-time adaptive LOS guidance framework for 3D path tracking for AUVs. By combining fixed-time stability theory with LOS guidance, this method guarantees path tracking convergence within a preset time range, with the convergence time independent of initial conditions. Furthermore, a fixed-time adaptive estimator is developed to rapidly compensate for time-varying sideslip disturbances caused by ocean currents. A time-varying look-ahead mechanism is also introduced to improve tracking performance on curved paths. Lyapunov analysis proves the fixed-time stability of the proposed framework, and numerical simulations and physical experiments demonstrate that, compared to state-of-the-art adaptive LOS methods, this framework exhibits superior tracking accuracy, convergence speed, and anti-interference capability. In simulation, the time-varying look-ahead variant reduced cross-track and vertical-track RMSE by 69.37\% and 67.46\%, respectively, during curved-path tracking. In field experiments with an Iver 3 AUV, the proposed fixed-time guidance reduced average tracking error by 56.35\% in straight-path evaluation and 27.59\% in curved-path evaluation compared with conventional adaptive LOS guidance.The proposed method provides a practical guidance-level solution for achieving reliable autonomous navigation of AUVs in complex marine environments.}
\end{abstract}



\maketitle


\section{Introduction}

Autonomous underwater vehicles (AUVs) are increasingly used in critical marine missions such as ocean exploration, infrastructure inspection, and surveillance \cite{UUV_background}. The success of these missions largely depends on reliable guidance, navigation, and control (GNC) systems capable of maintaining precise tracks in uncertain and dynamic underwater environments. Unlike land robots, AUVs operate in environments subject to strong and time-varying disturbances, particularly ocean currents, which can significantly degrade their path-tracking performance. Therefore, achieving rapid and reliable disturbance recovery is a fundamental requirement for practical autonomous underwater operations.

Line-of-sight (LOS) guidance \cite{LOS0} remains one of the most widely adopted path-following strategies due to its simplicity and effectiveness. However, traditional LOS methods typically rely on asymptotic convergence, meaning that while the tracking error gradually decreases, convergence cannot be guaranteed within a predictable time frame. This limitation is particularly critical in time-sensitive robotic tasks, as convergence delays can lead to navigation failures, excessive energy consumption, or collision risks.

To improve robustness, existing studies have focused primarily on refining the inner-loop attitude controller that tracks LOS-generated commands. For example, sliding mode control \cite{Robust3}, extended state observers \cite{Robust2}, and prescribed performance control \cite{Robust1} have been integrated with LOS guidance to compensate for disturbances. However, many commercial AUVs, such as IVER \cite{UUV1}, Remus \cite{UUV2}, and SLOCUM \cite{UUV3}, are equipped with built-in attitude controllers that are inaccessible to end users. Therefore, modifying the low-level controller is often impractical, prompting the development of more robust and intelligent guidance strategies that can be implemented directly at the guidance layer.

Adaptive LOS (ALOS) and integral LOS (ILOS) methods address environmental perturbations by estimating or compensating for sideslip effects \cite{iLOS,LOS3,ALOS-2D,Fossen2024}. These methods achieve stable path following under sustained perturbations and have been extended to 3D trajectories. However, their primary reliance on asymptotic convergence limits their applicability to tasks requiring rapid maneuvers and predictable recovery after severe perturbations.

Fixed-time control (FxTC) theory offers a promising solution that ensures system convergence within a predetermined time regardless of initial conditions \cite{FxTC2,FxTC1}. This property is particularly important for autonomous robots operating in uncertain environments, as recovery time directly impacts mission safety. FxTC has been successfully applied to AUV motion control problems \cite{FxTCex3,FxTCex4,FxTCex5}, but its integration with LOS guidance remains limited. Existing fixed-time LOS (FxTLOS) studies mainly focus on 2D path following of surface vehicles or simplified AUV scenarios \cite{FxTLOS1,FxTLOS3,FxTLOS2}, lacking comprehensive analysis of perturbation robustness, experimental verification, and extensions to 3D path tracking, where coupled pitch and yaw dynamics introduce additional challenges.

To overcome these limitations, this paper proposes a robust adaptive fixed-time line-of-sight (AFxTLOS) guidance framework for 3D path tracking of AUVs. This method offers advantages such as fast convergence, fixed-time sideslip perturbation estimation, and compatibility with existing AUV motion controllers. Furthermore, for curved path tracking, this paper proposes an improved version incorporating time-varying look-ahead distance. The main contributions are summarized below:

{\begin{itemize}
    \item A novel FxTLOS guidance framework for 3D path following of AUVs with guaranteed bounded convergence, independent of initial conditions of the system, while remaining compatible with commercial AUV motion controllers. 
    \item A fixed-time sideslip estimator that compensates for unknown environmental disturbances to enhance the performance of the proposed control system in dynamic environments. 
    \item Extensive validation through simulation and field experiments demonstrating improved tracking accuracy, faster convergence, and better robustness relative to state-of-the-art adaptive LOS method. 
\end{itemize}
}

This paper is structured as follows. Section \ref{Preliminaries} provides the necessary notations and lemmas and elaborates the problem formulation. Section \ref{Method} proposed the adaptive robust FxTLOS and its variation with time-varying look-ahead distance for tracking curved paths. Section \ref{Result} validates the proposed methodology in both numerical simulations and physical experiments with a comparative study against the ALOS presented in \cite{Fossen2024}. Section \ref{Conclusion} concludes the paper.

\section{Preliminaries and Problem Formulation}\label{Preliminaries}
In this section, the notation and lemmas used in this paper, the AUV kinematic model then reformulated in the amplitude-phase form to simplify the controller design, and the corresponding path-following error dynamics are derived. 

\subsection{Notations}\label{notations}
The following notations are used in this paper. 
\begin{itemize}
    \item For the trigonometric functions in the matrices, this paper uses: $\mathrm{s}\bullet $ as $\sin{\bullet}$,  $ \mathrm{c} \bullet $  as $\cos{\bullet}$, $\mathrm{t}\bullet$ as $\tan{\bullet}$, and $\mathrm{sc} \bullet$ as $\sec{\bullet}$. 
    \item For variables $x,y\in\mathbb{R}$, the notation $\sig$ is a function defined as $\sig^x(y)=|y|^x\sign(y)$. 
    \item For variables $z,\mu\in\mathbb{R}$, the notation $\lceil{z,\mu}\rfloor$ is a function defined as $\lceil{z,\mu}\rfloor=\sig^{1+\frac{1}{\mu}}(z)+\sig(z)+\sig^{1-\frac{1}{\mu}}(z)$
\end{itemize}

\subsection{Lemmas}
To prove controller stability, the following lemmas are used in this paper. 
\begin{lemma}\label{lemma2}
\cite{gao2020fixed}. Consider a nonlinear system
\begin{equation}\label{eq: nlfunc}
    \dot{\ve x}(t)=f\bigl(\ve x(t)\bigr),\quad t>t_0, \quad \ve x(t_0)=\ve x_0
\end{equation}
in which $\ve x = [x_1, ..., x_n]^T\in\mathbb{R}^n$ is the state variable and $f(\ve x):\:\mathbb{R}^n\rightarrow\mathbb{R}^n$ is a nonlinear function. It is assumed that the origin is an equilibrium point of the system, and in this paper, it is assumed that $t_0=0$ with $x_0$ as the initial condition. If there exists a Lyapunov function that satisfies
    \begin{equation}
        \dot{V}(x)\leq -\bigl(\alpha V^p(x)+\beta V^q(x)\bigr)^k
    \end{equation}
    with $\alpha,\beta, p, q, k\in\mathbb{R}>0$ and $pk<1,\:qk>1$, then the origin of the system is globally fixed-time stable, and its settling time $T$ is bounded by
    \begin{equation}
        T(\ve x_0) \leq T_{max} = \frac{1}{\alpha^k(1-pk)}+\frac{1}{\beta^k(qk-1)}, \:\forall \ve x_0\in\mathbb{R}^n
    \end{equation}
\end{lemma}
\begin{lemma}\label{lemma3}
    \cite{gao2020fixed}. For a nonlinear system \eqref{eq: nlfunc}, if there exists a Lyapunov function that satisfies
    \begin{equation}
        \dot{V}(x)\leq -(\alpha V^p(x)+\beta V^q(x))^k+\vartheta
    \end{equation}
    with $\alpha,\beta, p, q, k\in\mathbb{R}>0$, $pk<1,\:qk>1$, and $\vartheta\in(0,+\infty)$ is finite, then the system is practical fixed-time stable and the residual set of its solution is
    \begin{equation}\label{eq: fixed residual set}
    \begin{split}
        \mathcal{X}=\biggl\{\lim_{t\rightarrow T}\ve x|\|\ve x\|&\leq\min\Bigl\{\alpha^{-\frac{1}{p}}\Bigl(\frac{\vartheta}{1-\theta}\Bigr)^{\frac{1}{p}}, \\&\quad\quad \beta^{-\frac{1}{q}}\Bigl(\frac{\vartheta}{1-\theta}\Bigr)^\frac{1}{q}\Bigr\}\biggr\}
    \end{split}
    \end{equation}
    where $\theta\in(0,1)$, and the settling time $T$ is bounded by
    \begin{equation}
        T\leq T_{max}:=\frac{1}{\alpha\theta(1-p)}+\frac{1}{\beta\theta(q-1)}
    \end{equation}
\end{lemma}

\subsection{Kinematic model}
\begin{figure}
    \centering
    \includegraphics[width=1\linewidth]{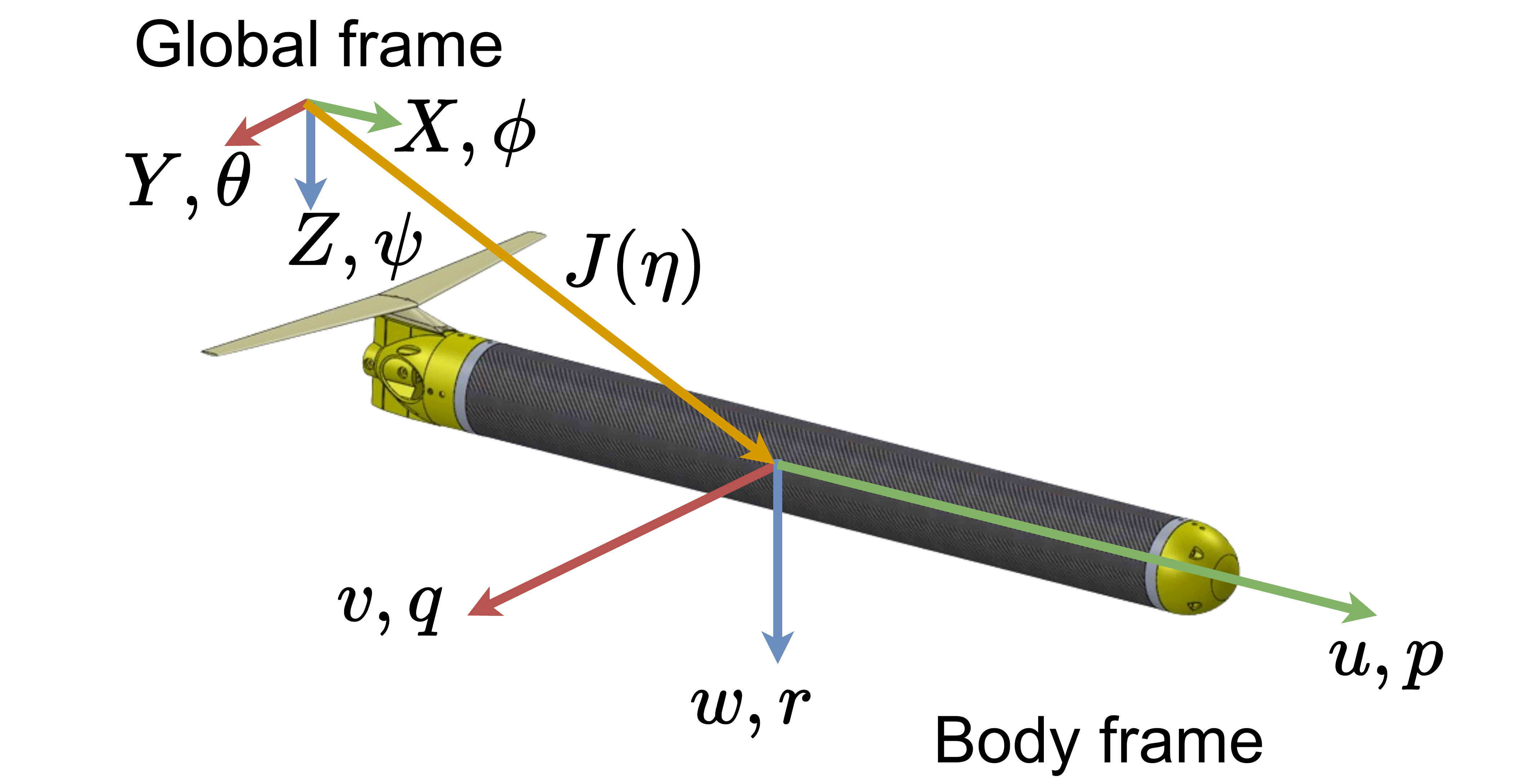}
    \caption{NED coordinate system of underwater vehicles. $X$,$Y$,$Z$,$\phi$,$\theta$,$\psi$ are the position and orientation of the vehicle in global frame, and $u$,$v$,$w$,$p$,$q$,$r$ are the velocity in body frame. } 
    \label{fig: Coord}
\end{figure}
The kinematic model of AUVs is described in a body-fixed coordinate frame and a global coordinate frame as shown in Fig. \ref{fig: Coord} and is given as 
\begin{equation} \label{eq: eom-kinematics2}
    \dot{\ve \eta} = \ve J(\ve \eta)\ve\nu
\end{equation}
where \(\ve\eta = [X,Y,Z]^T\) is the vehicle's position in the global frame, and \(\ve\nu = [u,v,w]^T \) is the vehicle's velocity in the body frame. The Euler angle rotation matrix $\ve J(\ve\eta)$ is defined in \textit{ZYX} convention
\begin{equation}
    \ve J(\ve\eta)=\begin{bmatrix}
        \mathbf{c}\theta\mathbf{c}\psi \quad& \mathbf{s}\phi\mathbf{s}\theta\mathbf{c}\psi-\mathbf{c}\phi\mathbf{s}\psi \quad& \mathbf{c}\phi\mathbf{s}\theta\mathbf{c}\psi+ \mathbf{s}\phi\mathbf{s}\psi \\
        \mathbf{c}\theta\mathbf{s}\psi \quad& \mathbf{s}\phi\mathbf{s}\theta\mathbf{s}\psi+\mathbf{c}\phi\mathbf{c}\psi \quad& \mathbf{c}\phi\mathbf{s}\theta\mathbf{s}\psi- \mathbf{s}\phi\mathbf{c}\psi \\
        -\mathbf{s}\theta \quad& \mathbf{s}\phi\mathbf{c}\theta \quad& \mathbf{c}\phi\mathbf{c}\theta
    \end{bmatrix}
\end{equation}

The kinematic model \eqref{eq: eom-kinematics2} can be rewritten in an amplitude-phase form \cite{Fossen2023AmplitudePhase} expressed by
\begin{equation}
    \label{eq: eom-kinematics-amplitudephase}
    \dot{\ve\eta}=\begin{bmatrix}
        U_h\cos{(\psi+\beta)}\\
        U_h\sin{(\psi+\beta)}\\
        -U_v\sin{(\theta-\alpha)}
    \end{bmatrix}
\end{equation}
where the phase angles are defined as
\begin{align}
    \alpha&=\tan^{-1}\Bigl(\frac{v\sin(\phi)+w\cos(\phi)}{u}\Bigr)\\
    \beta&=\tan^{-1}\Bigl(\frac{v\cos(\phi)-w\sin(\phi)}{U_v\cos(\theta-\alpha)}\Bigr)
\end{align}
and the amplitude speeds are given by
\begin{align}
    U_v&=u\sqrt{1+\tan^2(\alpha)}\\
    U_h&=U_v\cos(\theta-\alpha)\sqrt{1+\tan^2(\beta)}\label{eq: Uh_1}
\end{align}

\subsection{Error dynamics and control objective}
In a waypoint-based path following problem, the position tracking error is defined in an additional reference frame with an axis parallel to the straight path connecting two waypoints; for example, given a waypoint $\ve\eta_n=[X_n,Y_n,Z_n]^T$ and its next waypoint $\ve\eta_{n+1}=[X_{n+1},Y_{n+1},Z_{n+1}]^T$, the reference frame has its origin at $\ve\eta_{n}$ and the $x$-axis pointing toward $\ve\eta_{n+1}$. In this frame, the {along}-, {cross}-, and {vertical}-{track} error $\ve\eta_e=[X_e,Y_e,Z_e]^T$ is given by
\begin{equation}\label{eq: track error}
    \ve\eta_e=\ve{J}_Y^T \ve{J}_Z^T(\ve{\eta}-\ve{\eta}_n)
\end{equation}
in which the rotation matrices yield
\begin{align}
    \ve{J}_Y&=\begin{bmatrix}
        \mathbf{c}\pi_v&0&\mathbf{s}\pi_v\\
        0&1&0\\
        -\mathbf{s}\pi_v&0&\mathbf{c}\pi_v
    \end{bmatrix}\\
    \ve{J}_Z&=\begin{bmatrix}
        \mathbf{c}\pi_h&-\mathbf{s}\pi_h&0\\
        \mathbf{s}\pi_h&\mathbf{c}\pi_h&0\\
        0&0&1
    \end{bmatrix}
\end{align}
and the azimuth angle $\pi_h$ and elevation angle $\pi_v$ are given by
\begin{align}
    \pi_h&=\atan2(Y_{n+1}-Y_n, X_{n+1}-X_n)\\
    \pi_v&=\atan2\Bigl(-(Z_{n+1}-Z_n),\\\notag&\quad\quad\quad\sqrt{(X_{n+1}-X_n)^2+(Y_{n+1}-Y_n)^2}\Bigr)
\end{align}

Taking the derivative of \eqref{eq: track error} and substituting \eqref{eq: eom-kinematics-amplitudephase} yields the tracking error dynamics
\begin{equation}
    \label{eq: error dynamics}
    \dot{\ve\eta}_e = \ve{J}_Y^T \ve{J}_Z^T\begin{bmatrix}
        U_h\cos{(\psi+\beta)}\\
        U_h\sin{(\psi+\beta)}\\
        -U_v\sin{(\theta-\alpha)}
    \end{bmatrix}
\end{equation}
To track the target path in both vertical and horizontal planes, this work focuses on the cross- and vertical-track errors which can be expanded, from \eqref{eq: error dynamics}, as
\begin{align}
    \label{eq: cross-track error dynamics}
    \dot{Y}_e&=U_h\sin(\psi+\beta-\pi_h)\\
    \dot{Z}_e&=U_h\sin(\pi_v)\cos(\psi+\beta-\pi_h)-U_v\cos(\pi_v)\sin(\theta-\alpha)\label{eq: ze_1}
\end{align}
Substituting \eqref{eq: Uh_1} into \eqref{eq: ze_1} gives a compact form of the vertical-track error dynamics
\begin{equation}
\begin{split}
    \label{eq: vertical-track error dynamics}
    \dot{Z}_e &= -U_v\sin(\theta-\alpha-\pi_v)+U_v\sin(\pi_v)\cos(\theta-\alpha)\cdot\\&\quad\quad\quad\cdot\Bigl(\sqrt{1+\tan^2(\beta)}\cos(\psi+\beta-\pi_h)-1\Bigr)\\
    &=-U_v\sin(\theta-\alpha-\pi_v)+\frac{U_h\sin(\pi_v)}{\sqrt{1+\tan^2(\beta)}}\cdot\\&\quad\quad\quad\cdot\Bigl(\sqrt{1+\tan^2(\beta)}\cos(\psi+\beta-\pi_h)-1\Bigr)
\end{split}
\end{equation}

The control objective of this work is to design a guidance law that computes the required pitch and heading angles, $\theta_d$ and $\psi_d$, to enable the underwater vehicle to follow the preset waypoints and maintain its course, and to ensure the fixed-time stability of the cross- and vertical-track errors. During the path following along the straight segments connecting each waypoint, the following assumptions are made: 
\begin{assumption}\label{assum1}
    The underwater vehicle is moving at a nearly constant positive surge velocity s.t. $0<u_{min}\leq u\leq u_{max}$. 
\end{assumption}
\begin{assumption}\label{assum2}
    The crab angles $\alpha$ and $\beta$ are nearly constant along the straight path segments s.t. $\dot{\alpha}\approx0$,$\dot{\beta}\approx0$. 
\end{assumption}
\begin{assumption}\label{assum3}
    The crab angles $\alpha$ and $\beta$ are bounded along the entire path with known upper boundaries s.t. $|\alpha|\leq M_\alpha$ and $|\beta|\leq M_\beta$. 
\end{assumption}
\begin{assumption}\label{assum4}
    The lower-level controller for heading and pitch angle controls on the AUV guarantees accurate tracking with low time-delay in the response, s.t. $\theta(t)=\theta_d(t)$ and $\psi(t)=\psi_d(t)$. 
\end{assumption}

\begin{remark}
    {Assumptions~\ref{assum1}-\ref{assum3} are standard assumptions in ocean vehicle control because AUVs typically maintain a constant target yaw velocity during stable waypoint navigation, and ocean currents change slowly relative to vehicle dynamics, naturally limiting the yaw angle to a certain range and causing it to change slowly. Assumption \ref{assum4}, however, is an idealized assumption with practical limitations. In actual hardware deployments, the underlying attitude controller is strictly constrained by actuator limitations and the serial communication delay between the frontseat and backseat computers of the vehicle. Therefore, in such cases, the control system needs to be carefully tuned to ensure that the bandwidth of the attitude controller is fast enough to the guidance loop, thereby stabilizing the vehicle without causing dynamic coupling or trajectory overshoot.}
\end{remark}

\section{Controller Design}\label{Method}
A FxTLOS guidance law is designed as follows, 
\begin{align}
    \label{eq: FxTLOS-guidance-1}
    &\psi_d=\pi_h-{\beta}+\tan^{-1}\Bigl(\frac{-k_1(\lceil{Y_e,\mu}\rfloor)}{\Delta_h}\Bigr)\\
    \label{eq: FxTLOS-guidance-2}
    &\theta_d=\pi_v+{\alpha}+\tan^{-1}\Bigl(\frac{k_2(\lceil{Z_e,\mu}\rfloor)}{\Delta_v}\Bigr)
\end{align}
where $k_1, k_2>0\in\mathbb{R}$ and $\mu>1\in\mathbb{R}$ are controller parameters to be tuned, and $\Delta_h$ and $\Delta_v$ are look-ahead distances in the horizontal and vertical planes. 

Since not all AUV models have body-frame velocity feedback sensors, like Doppler Velocity Log (DVL), installed, in some cases the crab angles, $\alpha$ and $\beta$, are difficult to measure and compute. Even when a DVL is available, bottom-track measurements may become unavailable when the sensor loses bottom lock, and not all DVLs support water-tracking mode. An adaptive estimator therefore provides a valuable complementary source of information and a degree of redundancy when direct velocity measurements are unavailable or unreliable. To composensate for drift in both horizontal and vertical planes, a robust adaptation law is introduced to estimate the unknown crab angles, 
\begin{align}
    \label{eq: AFxTLOS-adaptation-1}
    &\dot{\hat{\alpha}}=\gamma_v\frac{\Delta_v}{\sqrt{\Delta_v^2+k_2^2(\lceil{Z_e,\mu}\rfloor)^2}}\proj{\bigl(\hat\alpha, k_2(\lceil{Z_e,\mu}\rfloor)\bigr)}\\
    \label{eq: AFxTLOS-adaptation-2}
    &\dot{\hat{\beta}}=\gamma_h\frac{\Delta_h}{\sqrt{\Delta_h^2+k_1^2(\lceil{Y_e,\mu}\rfloor)^2}}\proj{\bigl(\hat\beta, k_1(\lceil{Y_e,\mu}\rfloor)\bigr)}
\end{align}
where the parameter projection to enhance the controller's robustness is defined as
\begin{equation}
    \label{eq: AFxTLOS-projection}
    \proj{(\vartheta,\tau)}=\begin{cases}
         \bigl (1-c(\vartheta)\bigr )\tau&\text{if}\;|\vartheta|>M_\vartheta\;\text{and}\; \vartheta^T\tau>0  \\
         \quad\quad\tau& \text{else}
    \end{cases}
\end{equation}
in which $c(\vartheta)=\min\bigl\{(\hat{\vartheta}^2-M_\vartheta^2)/(M_{\hat{\vartheta}}^2-M_\vartheta^2),1\bigr\}$, $M_{\hat\vartheta}=M_\vartheta+\epsilon$, and $\epsilon>0\in\mathbb{R}$ is a small constant parameter to be designed such that $M_{\hat\vartheta}$ is slightly larger than $M_\vartheta$. Inserting \eqref{eq: AFxTLOS-adaptation-1} and \eqref{eq: AFxTLOS-adaptation-2} to \eqref{eq: FxTLOS-guidance-1} and \eqref{eq: FxTLOS-guidance-2} gives the adaptive version of the fixed-time LOS guidance law, 
\begin{align}
    \label{eq: AFxTLOS-guidance-proj-1}
    &\psi_d=\pi_h-\hat{\beta}+\tan^{-1}\Bigl(\frac{-k_1(\lceil{Y_e,\mu}\rfloor)}{\Delta_h}\Bigr)\\
    \label{eq: AFxTLOS-guidance-proj-2}
    &\theta_d=\pi_v+\hat{\alpha}+\tan^{-1}\Bigl(\frac{k_2(\lceil{Z_e,\mu}\rfloor)}{\Delta_v}\Bigr)
\end{align}
Substituting \eqref{eq: AFxTLOS-guidance-proj-1} and \eqref{eq: AFxTLOS-guidance-proj-2} into \eqref{eq: cross-track error dynamics} and \eqref{eq: vertical-track error dynamics} yields
    \begin{equation}
        \label{eq: proof-error-1}
        \begin{split}
        \dot{Y}_e&=U_h\sin{\biggl(\Tilde{\beta}+\tan^{-1}\Bigl(\frac{-k_1(\lceil{Y_e,\mu}\rfloor)}{\Delta_h}\Bigr)\biggr)}\\
        &=U_h\sin(\Tilde{\beta})\cos\biggl(\tan^{-1}\Bigl(\frac{k_1\sigmu{Y_e}}{\Delta_h}\Bigr)\biggr)\\&\quad\;-U_h\cos(\Tilde{\beta})\sin\biggl(\tan^{-1}\Bigl(\frac{k_1\sigmu{Y_e}}{\Delta_h}\Bigr)\biggr)\\
        &= U_h\frac{\Delta_h}{\sqrt{\Delta_h^2+k_1^2\sigmu{Y_e}^2}}\sin(\Tilde{\beta}) \\&\quad\;-U_h\frac{k_1\sigmu{Y_e}}{\sqrt{\Delta_h^2+k_1^2\sigmu{Y_e}^2}}\cos(\Tilde{\beta}) \\
        &= -U_h\frac{\Delta_h}{\sqrt{\Delta_h^2+k_1^2\sigmu{Y_e}^2}}\biggl[ \cos(\Tilde{\beta})\frac{k_1\sigmu{Y_e}}{\Delta_h}-\sin(\Tilde{\beta}) \biggr]
        \end{split}\end{equation}
        \begin{equation}
        \label{eq: proof-error-2}
        \begin{split}
        \dot{Z}_e&=-U_v\sin\biggl(\tan^{-1}\Bigl(\frac{k_2(\lceil{Z_e,\mu}\rfloor)}{\Delta_v}\Bigr)-\Tilde{\alpha}\biggr)+d(\beta,Y_e)\\
        &=-U_v\sin\biggl(\tan^{-1}\Bigl(\frac{k_2\sigmu{Z_e}}{\Delta_v})\Bigr)\biggr)\cos(\Tilde{\alpha}) \\&\quad\; + U_b\cos\biggl(\tan^{-1}\Bigl(\frac{k_2\sigmu{Z_e}}{\Delta_v})\Bigr)\biggr)\sin(\Tilde{\alpha})+d(\beta,Y_e)  \\
        &= -U_v\frac{k_2\sigmu{Z_e}}{\sqrt{\Delta_v^2+k_2^2\sigmu{Z_e}^2}}\cos(\Tilde{\alpha}) \\&\quad\; + U_v\frac{\Delta_v}{\sqrt{\Delta_v^2+k_2^2\sigmu{Z_e}^2}}\sin(\Tilde{\alpha}) + d(\beta,Y_e)  \\
        &= -U_v\frac{\Delta_v}{\sqrt{\Delta_v^2+k_2^2\sigmu{Z_e}^2}}\biggl[ \cos(\Tilde{\alpha})\frac{k_2\sigmu{Z_e}}{\Delta_v}-\sin(\Tilde{\alpha}) \biggr] \\&\quad\; + d(\tilde\beta, Y_e)
        \end{split}
    \end{equation} 
    where $\Tilde{\alpha} = \alpha - \hat\alpha$ and $\Tilde{\beta} = \beta - \hat\beta$ are the crab angle estimate errors, and the perturbation term is given by 
    \begin{equation}\begin{split}
    \label{eq: perturbation}
        d(\tilde\beta, Y_e) &= \frac{U_h\sin(\pi_v)}{\sqrt{1+\tan^2(\beta)}}\biggl[\sqrt{1+\tan^2(\beta)}\cdot\\&\quad\;\cdot\cos\biggl(\Tilde{\beta}-\tan^{-1}\Bigl(\frac{-k_1(\lceil Y_e,\mu\rfloor}{\Delta_h}\Bigr)\biggr)-1\biggr]
    \end{split}\end{equation}
Based on Assumption \ref{assum2}, $\dot{\tilde\alpha}=-\dot{\hat\alpha}$ and $\dot{\tilde\beta}=-\dot{\hat\beta}$, so, 
\begin{align}
\label{eq: proof-adaptation-1}
    \dot{\tilde\alpha} &= -\gamma_v\frac{\Delta_v}{\sqrt{\Delta_v^2+k_2^2(\lceil{Z_e,\mu}\rfloor)^2}}\proj{\bigl(\hat\alpha, k_2(\lceil{Z_e,\mu}\rfloor)\bigr)} \\
\label{eq: proof-adaptation-2}
    \dot{\tilde\beta} &= -\gamma_h\frac{\Delta_h}{\sqrt{\Delta_h^2+k_1^2(\lceil{Y_e,\mu}\rfloor)^2}}\proj{\bigl(\hat\beta, k_1(\lceil{Y_e,\mu}\rfloor)\bigr)}
\end{align}

\begin{lemma}\label{lemma: stability no adaptation}
    The origin $(Y_e, Z_e)=(0,0)$ of the subsystem consisting of \eqref{eq: proof-error-1} and \eqref{eq: proof-error-2} is fixed-time stable if the adaptation laws \eqref{eq: AFxTLOS-adaptation-1} and \eqref{eq: AFxTLOS-adaptation-2} estimates the crab angles perfectly, i.e. $\tilde\alpha=0$ and $\tilde\beta=0$, and if the perturbation term \eqref{eq: perturbation} is zero. Additionally, if the perturbation term is nonzero, then $Y_e=0$ is fixed-time stable and $Z_e=0$ is practical fixed-time stable. 
\end{lemma}
\begin{proof}
    See Appendix \ref{app: proof no adaptation}. 
\end{proof}
\begin{remark}
    Lemma \ref{lemma: stability no adaptation} also indicates that if the body-frame sensor feedback is available, with which the crab angles can be directly computed, the adaptation law can be isolated from the proposed robust controller structure. 
\end{remark}

\begin{lemma}\label{lemma: stability adaptation}
    The adaptive control law consisting of \eqref{eq: AFxTLOS-adaptation-1}, \eqref{eq: AFxTLOS-adaptation-2}, \eqref{eq: AFxTLOS-guidance-proj-1} and \eqref{eq: AFxTLOS-guidance-proj-2} guarantees stability of the tracking errors and convergence of the estimation errors in the presence of zero perturbation. 
\end{lemma}
\begin{proof}
    See Appendix \ref{app: proof adaptation}. 
\end{proof}

\begin{theorem}\label{theorem: convergence}
    Consider the AUV kinematic system described by \eqref{eq: eom-kinematics-amplitudephase}. Applying the proposed AFxTLOS control laws designed as \eqref{eq: AFxTLOS-guidance-proj-1} and \eqref{eq: AFxTLOS-guidance-proj-2} and the adaptation laws designed as \eqref{eq: AFxTLOS-adaptation-1} and \eqref{eq: AFxTLOS-adaptation-2} enables the cross- and vertical-track errors, the dynamic model shown in \eqref{eq: cross-track error dynamics} and \eqref{eq: vertical-track error dynamics}, to converge to a small neighborhood around the origin within a fixed time. Moreover, the finite convergence time $T_{max}$ is independent of the initial conditions and is upper bounded. 
\end{theorem}
\begin{proof}
    See Appendix \ref{proof: convergence}. 
\end{proof}

An alternative version of AFxTLOS involving a time-varying look-ahead distance is also proposed in this paper. The constant distance term $\Delta$ in \eqref{eq: AFxTLOS-guidance-proj-1} and \eqref{eq: AFxTLOS-guidance-proj-2} are replaced with an exponential function:
\begin{align}
    \Delta_h &= (\Delta_{h_{max}}-\Delta_{h_{min}})\exp{(-k_{\Delta_y}Y_e^2)}+\Delta_{h_{min}}\\
    \Delta_v &= (\Delta_{v_{max}}-\Delta_{v_{min}})\exp{(-k_{\Delta_z}Z_e^2)}+\Delta_{v_{min}}
\end{align}
in which in both vertical and horizontal planes, $\Delta_{max}$ and $\Delta_{min}$ are the maximum and minimum desired values of the look-ahead distance $\Delta$, respectively, and $k_\Delta>0\in\mathbb{R}$ is a design parameter. The basic strategy of the time-varying look-ahead distance mechanism is to assign a small distance value when the vehicle is far from the reference path, thereby making the path tracking more aggressive, and to have a large look-ahead distance if the vehicle is close to the target path, thus lowering the potential overshoots in tracking. The stability in Theorem \ref{theorem: convergence} can still be guaranteed for the AFxTLOS with time-varying look-ahead distance since the exponentially varying distances are bounded.

\section{Results}\label{Result}
This section evaluates the performance of the proposed guidance law and adaptation law via numerical simulations and field tests. The proposed control methodology is compared with the ALOS scheme presented in \cite{Fossen2024} to validate its convergence rate and path tracking accuracy. 

\subsection{Simulation}
The numerical simulation environment is built upon the \textit{Marine Systems Simulator} \cite{MSS} and utilizes a REMUS-100 AUV \cite{UUV2}, which is kept at a nearly constant speed of $2\:\si{m/s}$. The vehicle has a PID controller implemented to regulate the pitch and yaw angles by adjusting the angles of the vertical and horizontal rudders:
\begin{equation}
\begin{split}
    \delta_v = -k_{p_v}\wrap{(\theta-\theta_d)}&-k_{d_v}(\dot{\theta}-\dot{\theta}_d)\\&-k_{i_v}\int_0^t{\wrap(\theta-\theta_d)}d\tau
    \label{eq:vertical_pid} 
\end{split}
\end{equation}
\begin{equation}
\begin{split}
    \delta_h = -k_{p_h}\wrap{(\psi-\psi_d)}&-k_{d_h}(\dot{\psi}-\dot{\psi}_d)\\&-k_{i_h}\int_0^t{\wrap(\psi-\psi_d)}d\tau
    \label{eq:horizontal_pid}
\end{split}
\end{equation}
where $\wrap(\cdot)$ is a function that wraps angles to the interval $[-\pi,\pi)$, defined as
\begin{equation}
    \wrap(\theta)= [(\theta+\pi)\;\bmod\; 2\pi]-\pi
\end{equation}
The controller gains have been tuned to track the target attitudes in the simulator, and the numerical study in this work is built using the default settings of the pitch and yaw controllers. 

\begin{table*}
    \caption{Controller setups of ALOS, AFxTLOS, and TVLD-AFxTLOS guidance in simulation}
    \centering
    \begin{tabular}{||c|c||c|c|c||}
    \hline
         &Notation& ALOS & AFxTLOS & TVLD-AFxTLOS\\
         \hline\hline
         $\Delta_h,\Delta_v$ & Look-ahead distance & 20, 20 & 20, 20 & N/A \\
         \hline
         $\gamma_h, \gamma_v$ & Adaptive gain & 0.002,0.002 & 0.002, 0.002 & 0.002, 0.002\\
         \hline
         $k_1, k_2$ & Fixed-time controller gain & N/A & 1, 1 & 1, 1\\
         \hline
         $\mu$ & Fixed-time function parameter & N/A & 2 & 2\\
         \hline 
         $k_{\Delta_y},k_{\Delta_z}$ & Time-varying look-ahead distance gain & N/A & N/A & 1, 1 \\
         \hline 
         $\Delta_{h_{max}},\Delta_{h_{min}},\Delta_{v_{max}},\Delta_{v_{min}}$ & Time-varying look-ahead distance bound & N/A & N/A & 30, 10, 30, 10 \\
    \hline
    \end{tabular}
    \label{tab:controller setup}
\end{table*}

\begin{figure}
    \centering
    \includegraphics[width=1\linewidth]{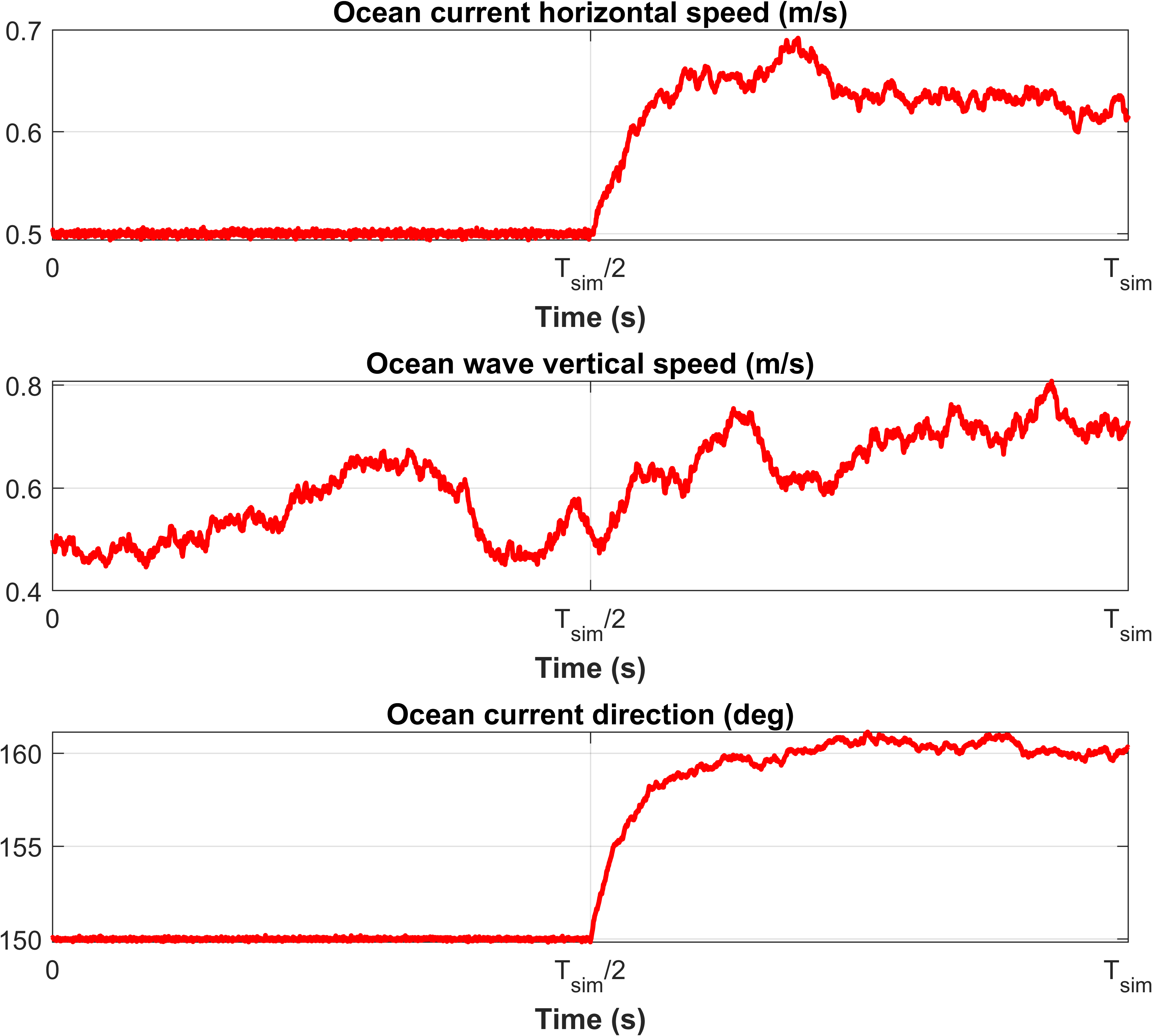}
    \caption{Time-varying environmental disturbances applied to the simulated AUV in the numerical study. }
    \label{fig:Ocean current}
\end{figure}

This study uses the ALOS method presented in \cite{Fossen2024} as the baseline and examines the proposed AFxTLOS guidance controller and its variation with time-varying look-ahead distance (TVLD-AFxTLOS). The parameter setups of the controllers are listed in Table. \ref{tab:controller setup}. When following the reference path defined by a series of waypoints, the guidance logic advances to the next waypoint as the vehicle enters the predefined neighborhood of the current target. Such a neighborhood is determined by
\begin{equation}
    \sqrt{(X-X_{n})^2+(Y-Y_{n})^2+(Z-Z_{n})^2}\leq R
     \label{eq:wpt_acceptance_R}
\end{equation}
where $R>0\in\mathbb{R}$ is the target acceptance radius. To further validate the robustness of the proposed method in the presence of environmental perturbations, currents and waves are added to impact the vehicle's stability. As shown in Fig. \ref{fig:Ocean current}, a horizontal current with a constant speed of $0.5\:\si{m/s}$ is initially added to the environment at a constant angle of $150$ degrees. After running half of the simulation, the current is turned to time-varying signal by adding a normally distributed noise with an average value of $0.15 \:\si{m/s}$ and standard deviation of $0.02$, and its direction is changed to be normally distributed with an average value of $160$ degrees and standard deviation of $0.1$. In the vertical plane, the ocean wave is simulated as a uniformly distributed random signal ranging in $[0.4,0.8]\:\si{m/s}$. With this setup, the simulation examines the controller's robustness in the presence of both constant unidirectional disturbances and time-varying perturbations. The simulations run at a frequency of $10 \si{Hz}$. Two scenarios are set to test the controllers' capability and robustness of tracking straight and curved paths.

\subsubsection{Straight path tracking}

\begin{figure*}[!ht]
    \centering
    \begin{subfigure}[t]{0.32\textwidth}
        \centering
        \includegraphics[width=\linewidth]{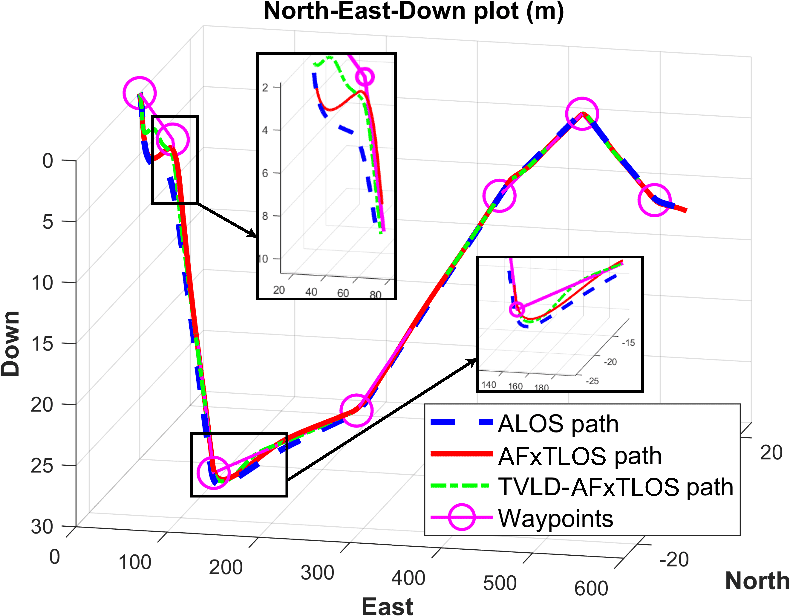}
        \caption{3D path tracking}
        \label{fig:straight 3d}
    \end{subfigure}
    \hfill
    \begin{subfigure}[t]{0.32\textwidth}
        \centering
        \includegraphics[width=\linewidth]{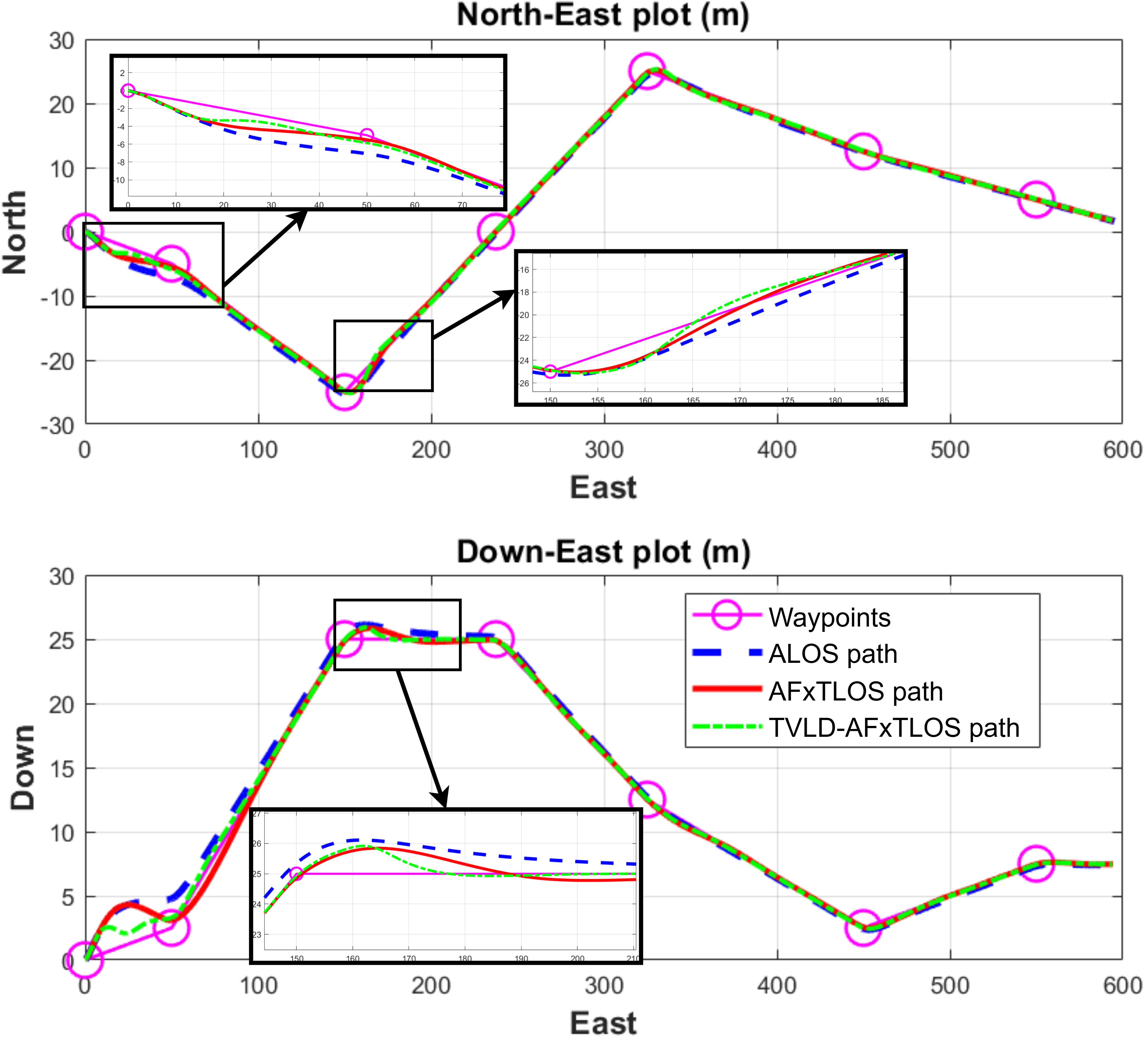}
        \caption{2D path tracking}
        \label{fig:straight 2d}
    \end{subfigure}
    \hfill
    \begin{subfigure}[t]{0.32\textwidth}
        \centering
        \includegraphics[width=\linewidth]{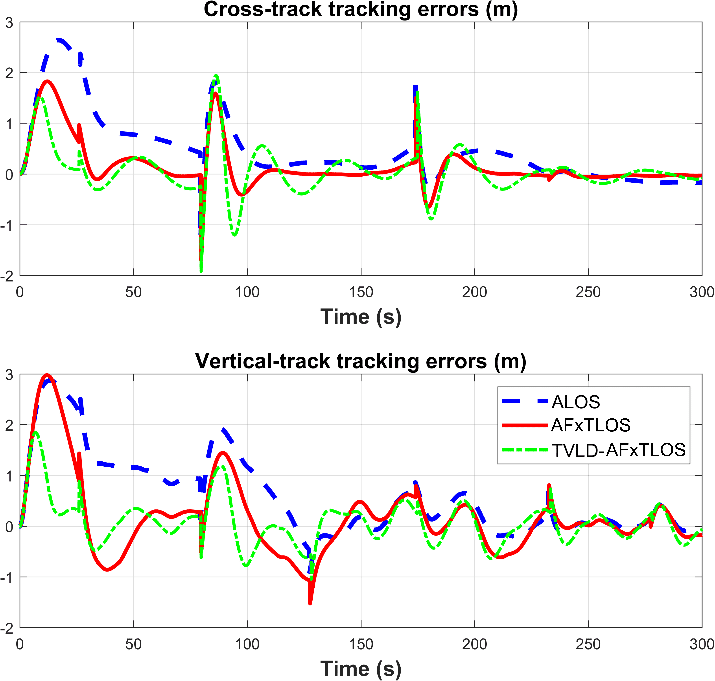}
        \caption{Tracking errors}
        \label{fig:straight error}
    \end{subfigure}

    \vspace{0.2cm} 

    \begin{subfigure}[t]{0.32\textwidth}
        \centering
        \includegraphics[width=\linewidth]{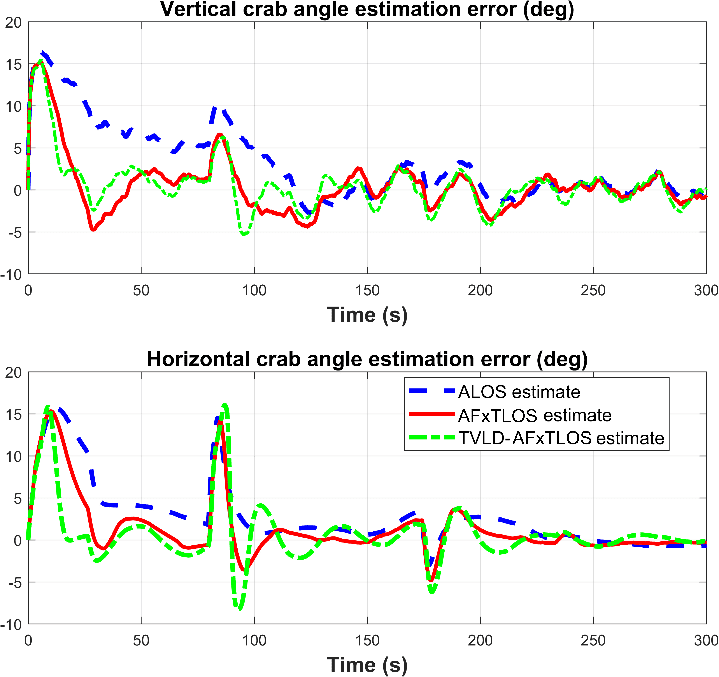}
        \caption{Crab angle adaptations}
        \label{fig:straight crab}
    \end{subfigure}
    \hfill
    \begin{subfigure}[t]{0.32\textwidth}
        \centering
        \includegraphics[width=\linewidth]{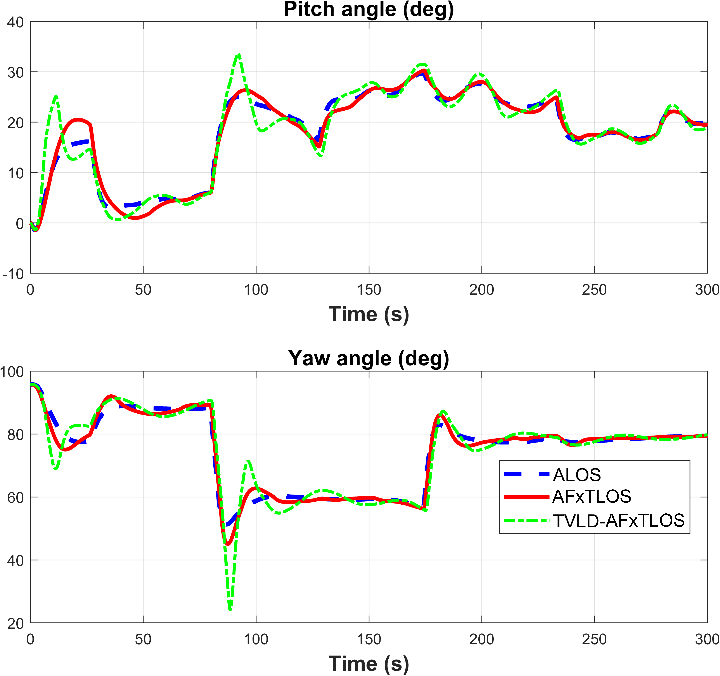}
        \caption{Control efforts (pitch and yaw)}
        \label{fig:straight pitch yaw}
    \end{subfigure}
    \hfill
    \begin{subfigure}[t]{0.32\textwidth}
        \centering
        \includegraphics[width=\linewidth]{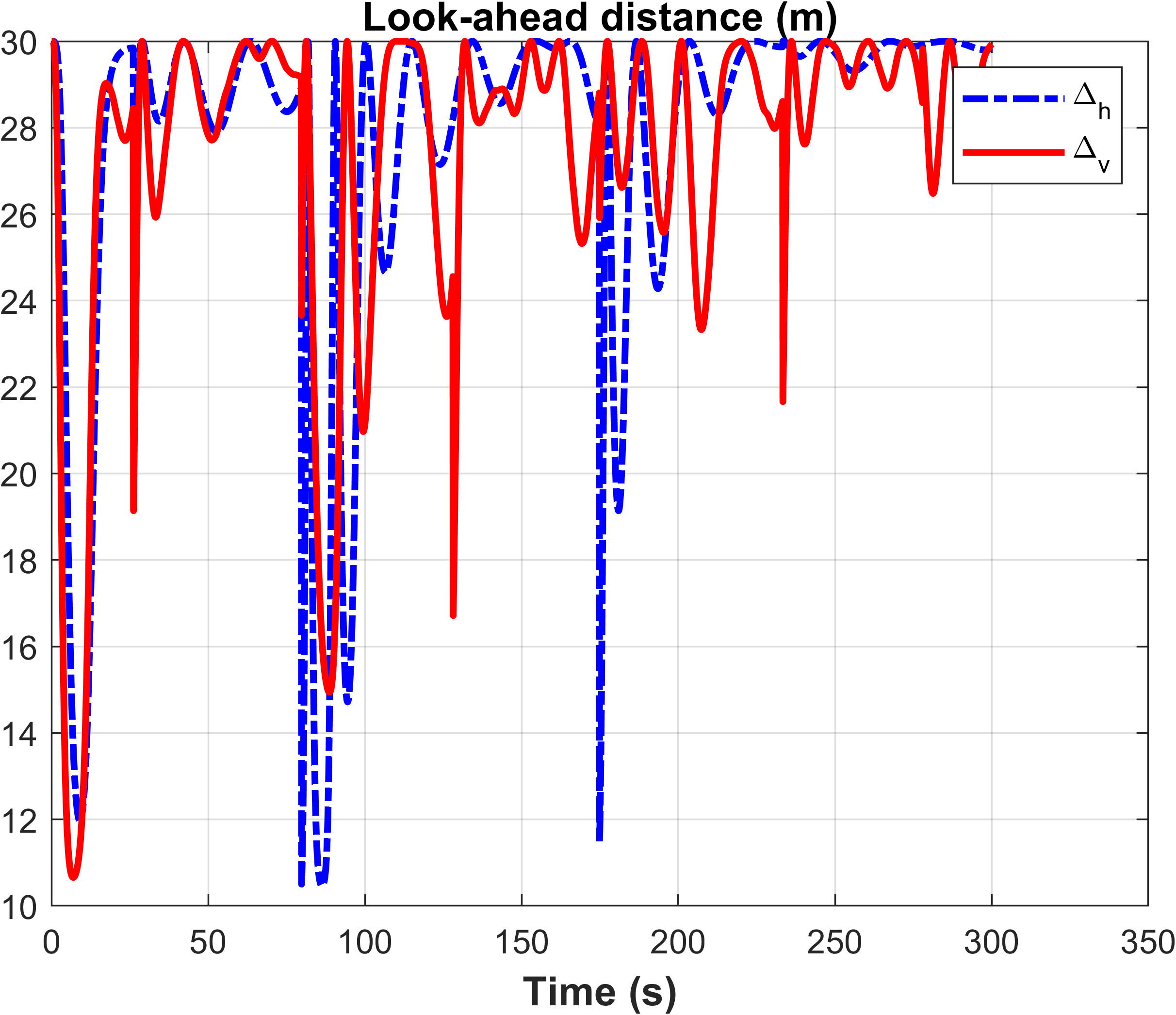}
        \caption{Look-ahead distances}
        \label{fig:straight look ahead}
    \end{subfigure}

    \caption{Performance evaluation of the straight path tracking scenario. (a) 3D spatial tracking demonstrates more precise tracking by the proposed fixed-time guidance scheme compared to ALOS. (b) 2D planar tracking highlights more rapid path convergence in both horizontal and vertical planes. (c) Cross- and vertical-track errors exhibit rapid, accurate convergence to zero due to the fixed-time control mechanism. (d) Crab angle estimations show faster and more precise angle compensation. (e) Pitch and yaw angle adjustments illustrate the more aggressive control efforts required by the fixed-time schemes. (f) Time-varying look-ahead distances gradually approach maximum allowable values as the tracking error converges to zero.}
    \label{fig:straight_combined}
\end{figure*}

The first simulation evaluates the convergence capability of the proposed guidance schemes for straight-line path following. The reference trajectory consists of a series of waypoints:
$[0,0,0]^T\rightarrow[-5,50,2.5]^T\rightarrow[-25,150,25]^T
\rightarrow[0,240,25]^T\rightarrow[25,325,12.5]^T
\rightarrow[12.5,450,2.5]^T\rightarrow[5,550,7.5]^T$.
The AUV starts from a stationary condition, $\eta(0)=\nu(0)=0$, and the waypoint acceptance radius is set to $R=4\:\si{m}$.

Figures \ref{fig:straight 3d}--\ref{fig:straight error} demonstrate that both fixed-time guidance schemes achieve faster path convergence than ALOS in both horizontal and vertical planes. While ALOS gradually reduces the tracking error due to its asymptotic convergence property, AFxTLOS and TVLD-AFxTLOS rapidly drive the cross-track and vertical-track errors toward zero. TVLD-AFxTLOS provides the highest tracking accuracy near waypoint transitions by adapting the look-ahead distance according to the tracking error. The RMSE comparison further confirms this improvement: ALOS results in cross-track and vertical-track errors of $0.81\:\si{m}$ and $0.99\:\si{m}$, respectively, whereas AFxTLOS reduces them to $0.47\:\si{m}$ and $0.79\:\si{m}$, and TVLD-AFxTLOS achieves $0.44\:\si{m}$ and $0.43\:\si{m}$.

The disturbance estimation performance is evaluated through crab angle adaptation. As shown in Fig. \ref{fig:straight crab}, the fixed-time adaptation law enables faster convergence and lower estimation errors compared with ALOS. Specifically, the RMSE values of horizontal and vertical crab angle estimation are reduced from $5.44^\circ$ and $4.72^\circ$ with ALOS to $3.46^\circ$ and $3.78^\circ$ with AFxTLOS, and further to $3.15^\circ$ and $3.63^\circ$ with TVLD-AFxTLOS. These improvements demonstrate that rapid disturbance estimation directly contributes to improved path tracking performance.

The desired pitch and yaw commands generated by different guidance laws are compared in Fig. \ref{fig:straight pitch yaw}. The fixed-time methods require more aggressive attitude adjustments to achieve faster convergence, indicating a trade-off between tracking accuracy and control effort. Although this may increase energy consumption, it provides improved transient performance for missions requiring rapid trajectory recovery.

The evolution of the adaptive look-ahead distances is shown in Fig.~\ref{fig:straight look ahead}. During the initial convergence phase, the horizontal and vertical look-ahead distances decrease to provide more aggressive guidance corrections when the tracking errors are large. As the vehicle approaches the desired path and the tracking errors converge, both look-ahead distances gradually increase toward their maximum allowable values. This adaptive behavior allows TVLD-AFxTLOS to achieve a balance between rapid convergence during transient motion and smooth path following during steady-state tracking.

\begin{figure*}[!ht]
    \centering
    \begin{subfigure}[t]{0.32\textwidth}
        \centering
        \includegraphics[width=\linewidth]{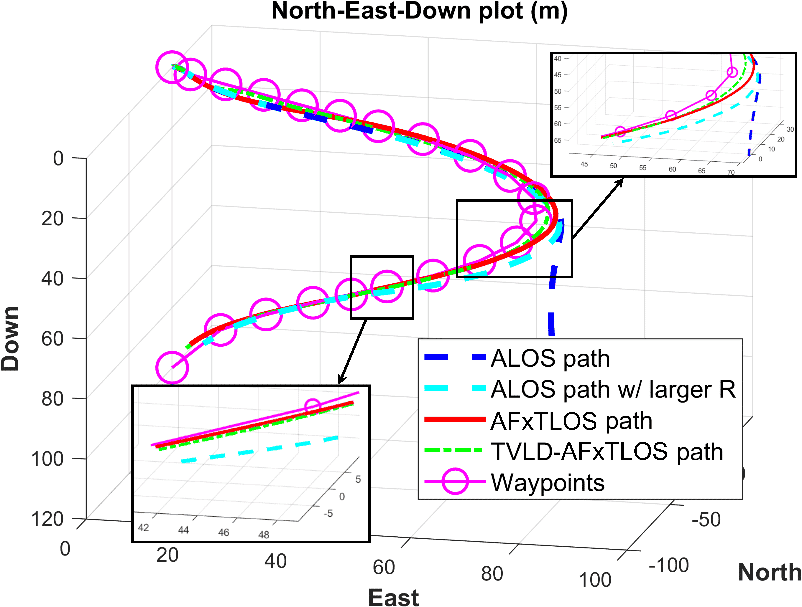}
        \caption{3D path tracking}
        \label{fig:curved 3d}
    \end{subfigure}
    \hfill
    \begin{subfigure}[t]{0.32\textwidth}
        \centering
        \includegraphics[width=\linewidth]{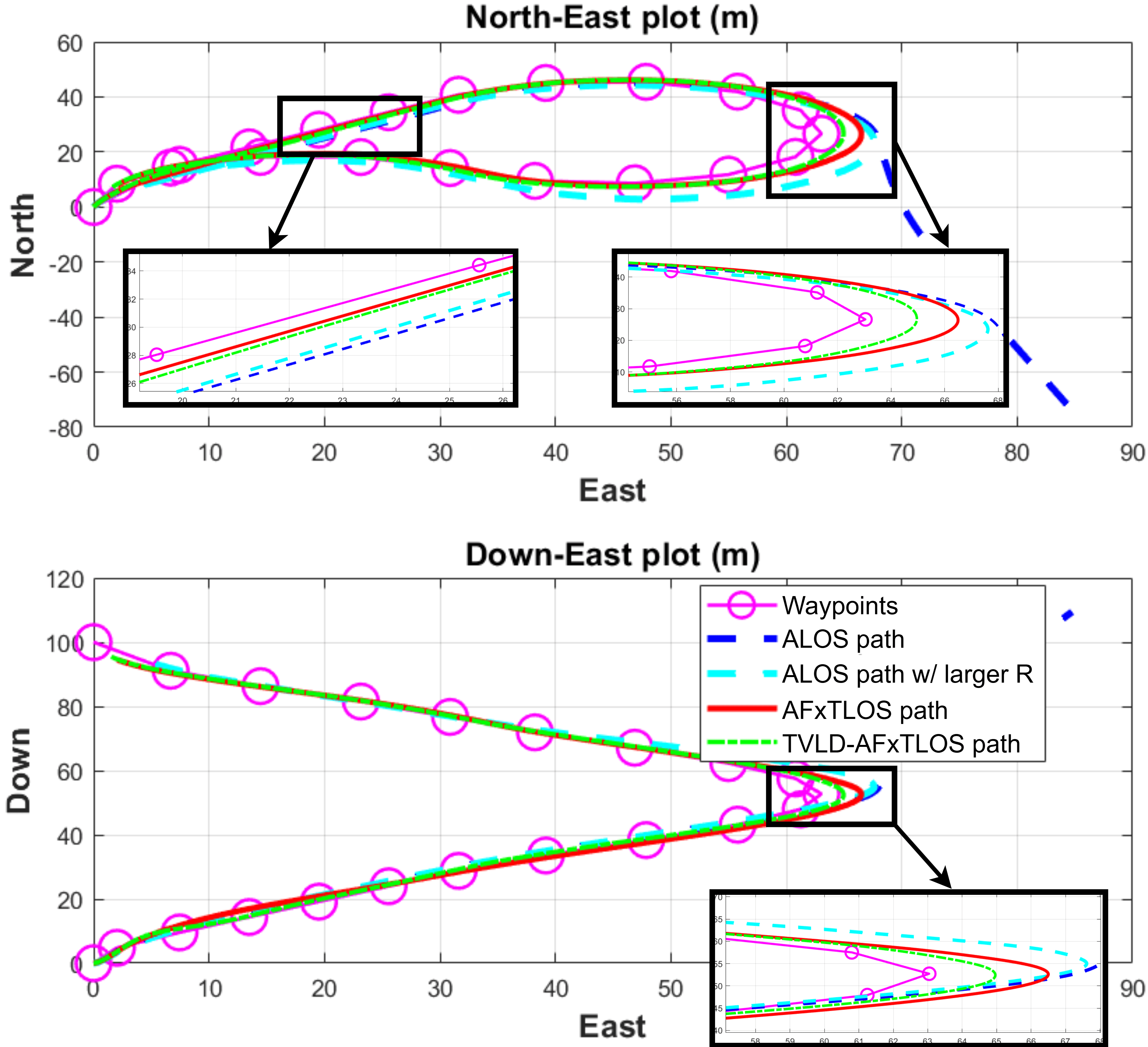}
        \caption{2D path tracking}
        \label{fig:curved 2d}
    \end{subfigure}
    \hfill
    \begin{subfigure}[t]{0.32\textwidth}
        \centering
        \includegraphics[width=\linewidth]{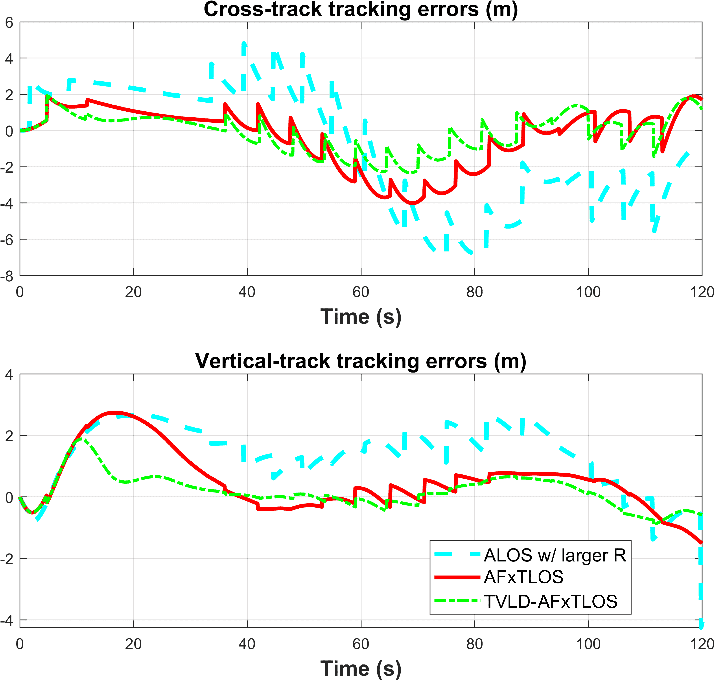}
        \caption{Tracking errors}
        \label{fig:curved error}
    \end{subfigure}

    \vspace{0.2cm} 

    \begin{subfigure}[t]{0.32\textwidth}
        \centering
        \includegraphics[width=\linewidth]{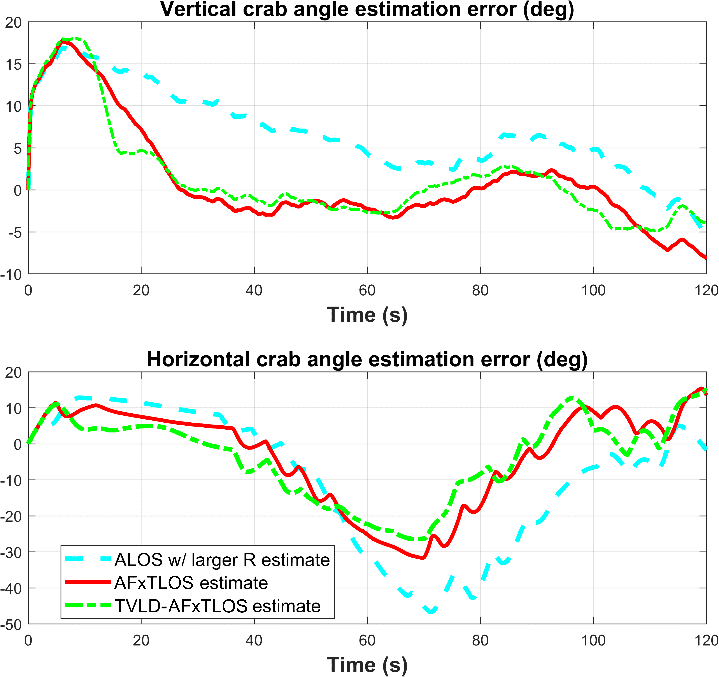}
        \caption{Crab angle adaptations}
        \label{fig:curved crab}
    \end{subfigure}
    \hfill
    \begin{subfigure}[t]{0.32\textwidth}
        \centering
        \includegraphics[width=\linewidth]{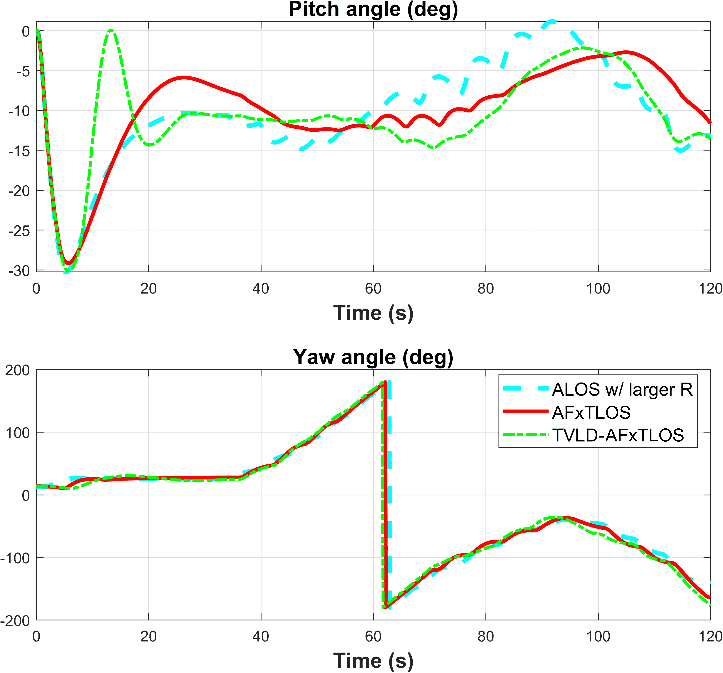}
        \caption{Control efforts (pitch and yaw)}
        \label{fig:curved pitch yaw}
    \end{subfigure}
    \hfill
    \begin{subfigure}[t]{0.32\textwidth}
        \centering
        \includegraphics[width=\linewidth]{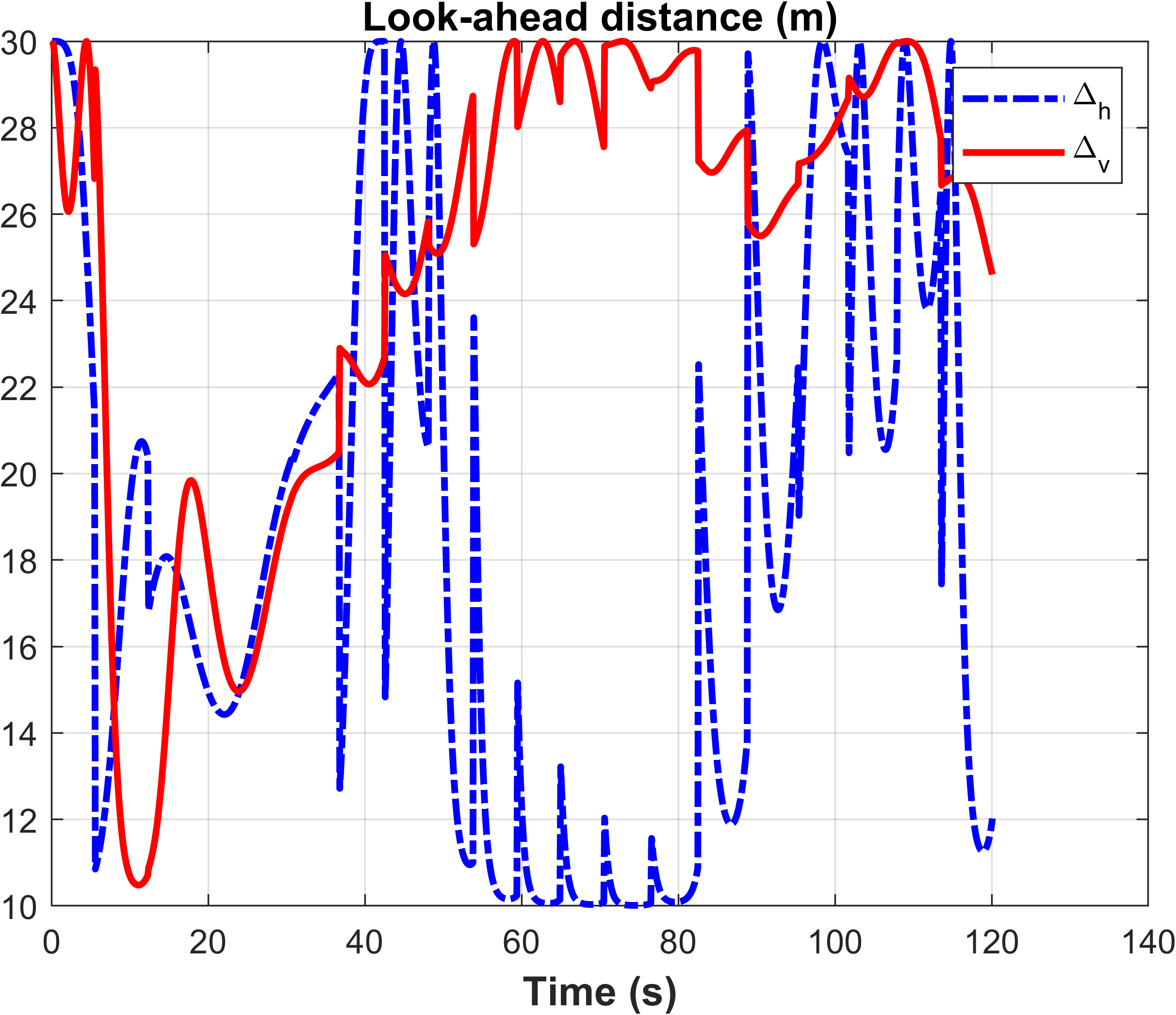}
        \caption{Look-ahead distances}
        \label{fig:curved look ahead}
    \end{subfigure}

    \caption{Performance evaluation of the curved path tracking scenario. (a) 3D spatial tracking demonstrates more precise tracking by the proposed fixed-time guidance scheme compared to the prior ALOS method. (b) 2D planar tracking shows more rapid path convergence in both horizontal and vertical planes. (c) Tracking errors illustrate the robust convergence properties of the fixed-time methods over ALOS. (d) Crab angle estimations demonstrate faster and more precise angle compensations in the vertical plane; horizontal adaptation exhibits chattering due to frequent waypoint switching, but still outperforms ALOS. (e) Control efforts show the more aggressive adjustments in pitch and yaw angles required by the fixed-time schemes. (f) Time-varying look-ahead distances reveal that vertical distance approaches maximum values upon convergence, whereas horizontal distance fluctuates between extremes due to rapid waypoint switching along the curve.}
    \label{fig:curved_combined}
\end{figure*}

\subsubsection{Curved path tracking}

The second simulation investigates the performance of the proposed methods in a more challenging curved-path scenario. The reference trajectory is generated using a Dubins path with a minimum turning radius of $20\:\si{m}$ and connects the initial position with a final waypoint $[0,0,100]^T$ while requiring a $180^\circ$ heading change. The waypoint acceptance radius is maintained at $R=4\:\si{m}$.

Figures \ref{fig:curved 3d}--\ref{fig:curved error} show that the advantage of fixed-time guidance becomes more significant during curved-path tracking. Due to accumulated disturbance effects and slow convergence, ALOS fails to maintain the desired trajectory under the original waypoint switching condition. Even after increasing its acceptance radius to $R=8\:\si{m}$, ALOS exhibits larger deviation from the reference path. In contrast, AFxTLOS and TVLD-AFxTLOS successfully complete the trajectory while maintaining bounded tracking errors. Among the proposed methods, TVLD-AFxTLOS achieves the best performance by dynamically adjusting the look-ahead distance, resulting in RMSE values of $1.04\:\si{m}$ and $0.56\:\si{m}$ for cross-track and vertical-track errors, respectively. AFxTLOS achieves comparable performance with RMSE values of $1.67\:\si{m}$ and $1.08\:\si{m}$, while ALOS results in significantly larger errors of $3.40\:\si{m}$ and $1.72\:\si{m}$.

The crab angle estimation results in Fig. \ref{fig:curved crab} further highlight the robustness of the fixed-time adaptation scheme under complex maneuvers. Compared with ALOS, the proposed methods provide faster disturbance compensation and lower estimation errors. The RMSE values of horizontal and vertical crab angle estimation are reduced from $8.46^\circ$ and $20.74^\circ$ with ALOS to $6.21^\circ$ and $13.87^\circ$ with AFxTLOS, and to $5.85^\circ$ and $11.95^\circ$ with TVLD-AFxTLOS. The larger estimation variation in the horizontal plane is caused by frequent waypoint switching and rapid yaw adjustments during turning maneuvers.

Figure \ref{fig:curved pitch yaw} compares the corresponding attitude commands. Similar to the straight-path case, the fixed-time methods generate more active control commands to achieve rapid error rejection. TVLD-AFxTLOS requires the largest pitch adjustments due to its improved trajectory-following accuracy, demonstrating the inherent trade-off between tracking precision and energy efficiency.

The corresponding look-ahead distance adaptation is illustrated in Fig.~\ref{fig:curved look ahead}. Unlike the straight-path case, the horizontal look-ahead distance exhibits frequent variations due to continuous waypoint transitions and curvature changes along the reference trajectory. By reducing the horizontal look-ahead distance during large cross-track deviations, TVLD-AFxTLOS generates stronger lateral corrections to improve path adherence during turning maneuvers. Meanwhile, the vertical look-ahead distance gradually increases as the vertical tracking error converges, resulting in smoother depth regulation. These results demonstrate that the time-varying look-ahead mechanism improves adaptability during complex 3D maneuvers.

\subsubsection{Simulation summary}

The simulation results demonstrate that fixed-time LOS guidance provides significant advantages over conventional adaptive LOS guidance for disturbed 3D AUV path following. The proposed methods achieve faster convergence, improved disturbance estimation, and reduced tracking errors without modifying the vehicle's internal attitude controller. The benefit of the proposed approach is particularly evident in curved-path tracking, where asymptotic convergence limits the performance of conventional ALOS. Compared with ALOS, TVLD-AFxTLOS reduces the tracking RMSE by $69.37\%$ in cross-track error and $67.46\%$ in vertical-track error, while reducing crab angle estimation RMSE by $32.27\%$ and $42.40\%$ for horizontal and vertical disturbances, respectively. The increased control activity required by fixed-time guidance represents a potential energy-efficiency trade-off, which will be considered in future optimization-based extensions.

\subsection{Field Experiments}

\subsubsection{Experimental setup}
To evaluate the proposed AFxTLOS guidance framework under realistic operating conditions, field experiments were conducted using an Iver 3 AUV (Fig.~\ref{fig:iver-pic}). The vehicle is a torpedo-shaped platform with a length of $2.16\:\si{m}$ and diameter of $15\:\si{cm}$, equipped with separate frontseat and backseat computers. Since the frontseat controller is a closed system responsible for low-level vehicle control, all proposed guidance algorithms were implemented on the backseat computer following the architecture previously used in underwater autonomy applications~\cite{IVER1,IVER2,IVER3}. The guidance module communicates with the frontseat controller through a serial NMEA interface by sending desired thrust, rudder, and elevator commands.

\begin{figure}
    \centering
    \includegraphics[width=1\linewidth]{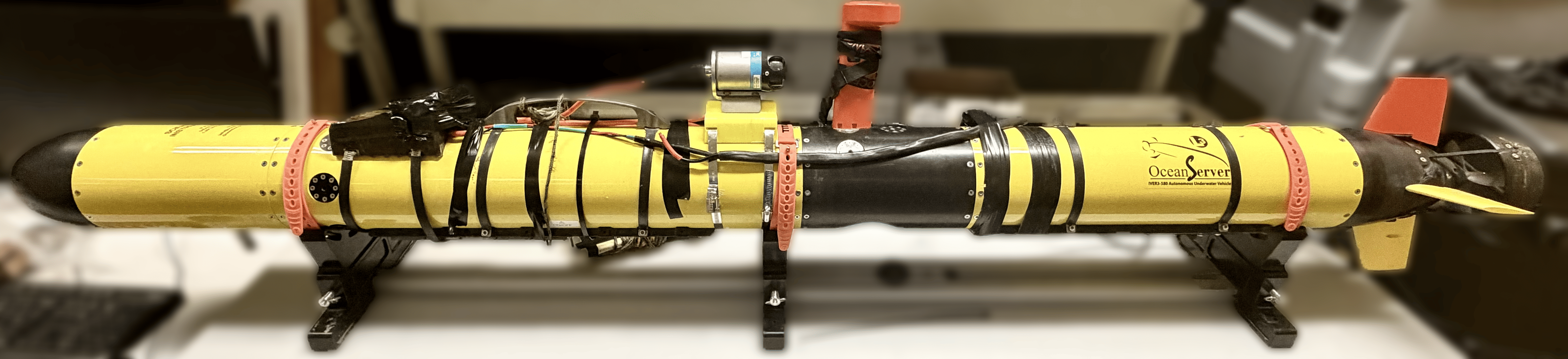}
    \caption{The Iver 3 AUV platform. }
    \label{fig:iver-pic}
\end{figure}

The vehicle is equipped with a DVL for underwater velocity estimation and onboard sensors for heading, depth, and position measurements. During underwater operation, the navigation estimate is maintained through dead reckoning and periodically corrected through surfacing events to obtain GPS measurements. The heading and depth loops are controlled by existing PID autopilots, which remain unchanged throughout all experiments. Therefore, the proposed method only modifies the guidance layer without requiring access to the internal vehicle controller.

The implementation was developed in ROS and executed at $5\:\si{Hz}$. The system consisted of a guidance node responsible for computing reference heading and pitch commands and a communication node responsible for sensor acquisition and command transmission. The controller parameters used for ALOS, AFxTLOS, and TVLD-AFxTLOS are summarized in Table~\ref{tab:controller_field_setup}.

{Validation consisted of three stages: (i) Gazebo-based verification of the ROS implementation and vehicle software integration, (ii) controller tuning during preliminary field deployments, and (iii) comparative field experiments evaluating ALOS, AFxTLOS, and TVLD-AFxTLOS. The final evaluation mission consisted of four waypoints, including surface start and end points and two intermediate underwater waypoints at a depth of $2\:\si{m}$. This trajectory generated three navigation segments involving descent, underwater tracking, and surfacing. During the ascent phase, the system acquired GPS signals for positioning to calibrate the position feedback of the vehicle control system. }

\subsubsection{{ROS software integration}}
The first stage of validation was conducted in simulation using Gazebo to verify the logic of the proposed guidance law and to evaluate its performance within the Iver3 software environment. The purpose of this stage was to validate the node architecture and overall system logic prior to field testing, reducing deployment effort and risk. An image of the Bicentennial Nature Area Lake along with its corresponding Gazebo simulation is shown in Fig. \ref{fig:real_sim_environments}.

\begin{figure}[!ht]
    \centering
    \includegraphics[width=1\linewidth]{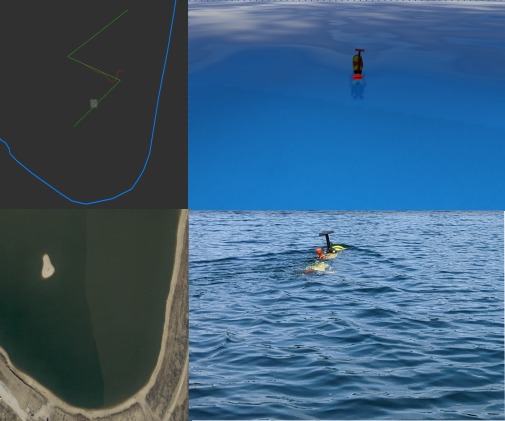}
    \caption{The top row illustrates the first stage of validation conducted prior to field deployments, using Gazebo and ROS for simulation-based testing. The bottom row shows the second stage of validation carried out at Fairfield Lakes and the Bicentennial Nature Area, where the majority of field deployments were performed.}
    \label{fig:real_sim_environments}
\end{figure}

\begin{table*}
    \caption{Setups of ALOS, AFxTLOS, and TVLD-AFxTLOS guidance in field tests}
    \centering
    \begin{tabular}{||c|c||c|c|c||}
    \hline
         &Notation& ALOS & AFxTLOS & TVLD-AFxTLOS\\
         \hline\hline
         $\Delta_h,\Delta_v$ & Look-ahead distance & 30, 30 & 30, 30 & N/A \\
         \hline
         $\gamma_h, \gamma_v$ & Adaptive gain & 0.002,0.002 & 0.002, 0.002 & 0.002, 0.002\\
         \hline
         $k_1, k_2$ & Fixed-time controller gain & N/A & 1, 1 & 1, 1\\
         \hline
         $\mu$ & Fixed-time function parameter & N/A & 2 & 2\\
         \hline 
         $k_{\Delta_y},k_{\Delta_z}$ & Time-varying look-ahead distance gain & N/A & N/Ah & 1, 1 \\
         \hline 
         $\Delta_{h_{max}},\Delta_{h_{min}},\Delta_{v_{max}},\Delta_{v_{min}}$ & Time-varying look-ahead distance bound & N/A & N/A & 50, 30, 50, 30 \\
    \hline
    \end{tabular}
    \label{tab:controller_field_setup}
\end{table*}

\subsubsection{{Controller tuning during field deployments}}
The second stage of validation was conducted at Fairfield Lakes and the Bicentennial Nature Area, both located in Indiana. This stage focused on field validation and controller tuning, including tuning of the depth and heading PID controllers under real environmental disturbances. These experiments also served to assess system robustness, sensor behavior, and overall closed-loop performance prior to formal comparative evaluation.


\subsubsection{{Comparative experiment 1: Moderate environmental disturbances}}
The first comparative experiment, conducted on Day 1, was performed under gentle breeze conditions, with an average wind speed of $4.6 \;\si{m/s}$ and gusts up to $6.7 \;\si{m/s}$. These conditions generated surface waves that significantly degraded the reliability of the GPS signal when the vehicle resurfaced. This behavior is visible Fig. \ref{fig:field_comparison_1030_horizontal} during the final leg of the trajectory, where localization becomes unstable upon surfacing. As a result, performance comparison during this segment becomes less reliable. However, the first two legs of the mission provide a consistent basis for evaluating controller performance.

From these segments, AFxTLOS demonstrates the best overall tracking performance, achieving rapid convergence to the predefined trajectory. In the first turn, the increased aggressiveness of AFxTLOS produces a slight overshoot, which delays convergence when transitioning to the second leg of the mission. {In contrast, the TVLD-AFxTLOS exhibits smoother behavior during turns and decreases the average tracking error from $4.49\:\si{m}$ by ALOS to $2.68\:\si{m}$. The variable look-ahead softens the response in these regions, reducing overshoot at the expense of a slightly slower alignment compared to AFxTLOS, which results in average tracking error of $1.96\:\si{m}$}.

In this run, the ALOS controller shows a more conservative turning behavior, leading to slower convergence toward the reference path. While this produces a smoother trajectory with limited overshoot, it delays alignment with the first and second legs.

Regarding the depth response, presented in Fig. \ref{fig:field_comparison_1030_vertical}, the differences among the controllers further highlight their transient characteristics. {ALOS results in an average tracking error in the vertical plane of $0.29\:\si{m}$, AFxTLOS leads to $0.48\:\si{m}$, and TVLD-AFxTLOS, $0.21\:\si{m}$}. During descent, all controllers achieve stable convergence toward the commanded $2\:\si{m}$ depth. However, the improved responsiveness of the AFxTLOS and TVLD-AFxTLOS controllers is accompanied by increased oscillatory behavior in the mid-depth segment, although the latter exhibits lower overshoot due to the varying looka-head. During ascent, minor oscillations are observed for all controllers, mainly because the vertical control input, $\delta_v$, must be adjusted to bring the vehicle to the surface for a GPS fix.

\begin{figure*}[!ht]
  \centering
    \begin{subfigure}[t]{0.48\textwidth}
        \centering
        \includegraphics[width=1.0\columnwidth]{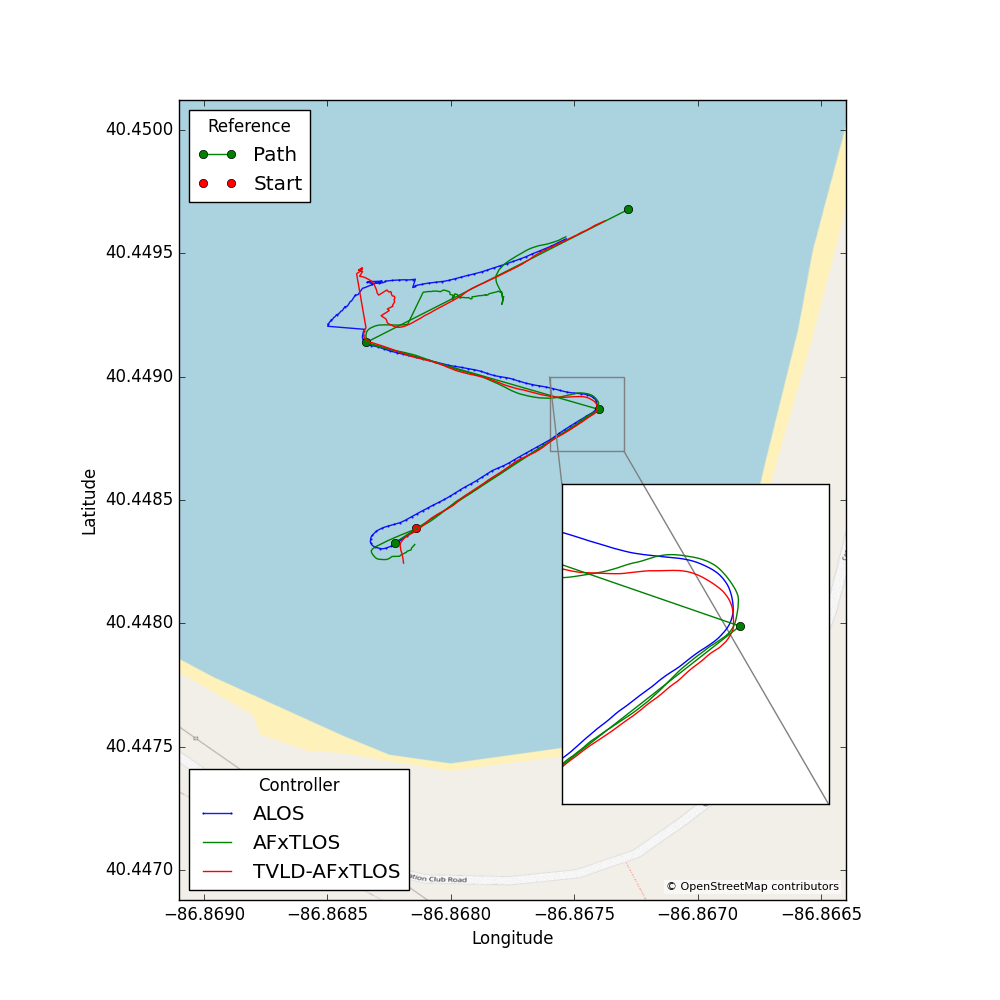}
        \caption{Horizontal trajectories}
        \label{fig:field_comparison_1030_horizontal}
    \end{subfigure}
    \hfill
    \begin{subfigure}[t]{0.48\textwidth}
        \centering
        \includegraphics[width=1.0\columnwidth]{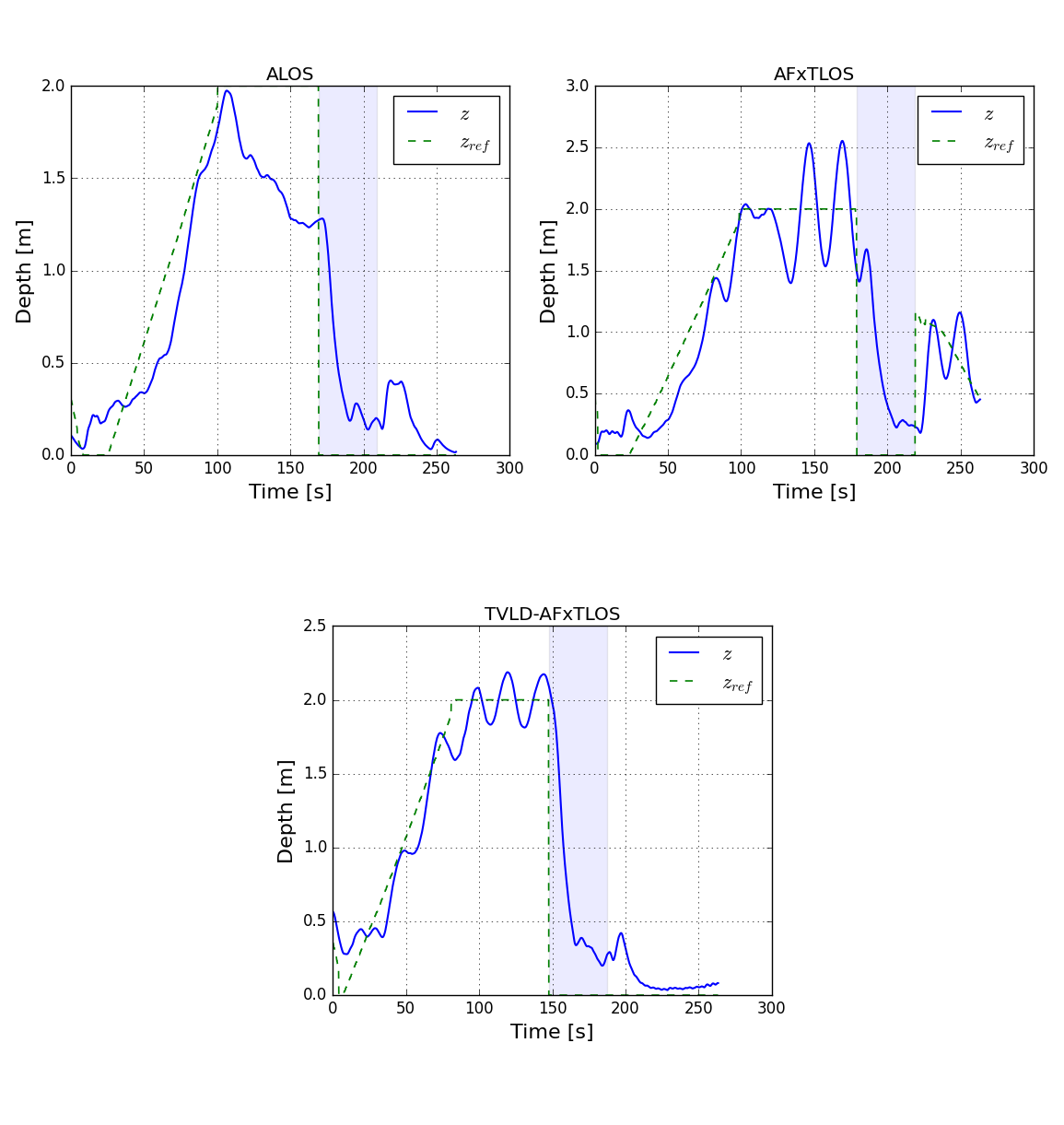}
        \caption{Depth profiles}
        \label{fig:field_comparison_1030_vertical}
    \end{subfigure}
    \caption{Left: Comparison of horizontal path-following for ALOS (blue), AFxTLOS (green), and TVLD-AFxTLOS (red). The AFxTLOS controller exhibits the best overall performance. The ALOS controller demonstrates good tracking capability with noticeable oscillatory behavior. The TVLD-AFxTLOS controller achieves the second-best performance. Right: Depth profiles corresponding to the same mission, illustrating the differences in transient response and convergence behavior across controllers.}
    \label{fig:field_comparison_1030}
\end{figure*}

The evolution of the crab angles, shown in Fig. \ref{fig:crab_1030}, provides further insight into the behavior of the adaptive guidance laws. In the vertical plane (see Fig. \ref{fig:crab_alpha_1030}), the adaptive estimate $\hat{\alpha}$ remains smooth across all controllers and captures the low frequency trend of $\alpha$, while attenuating high frequency fluctuations.

\begin{figure*}[!ht]
  \centering
    \begin{subfigure}[t]{0.48\textwidth}
        \centering
        \includegraphics[width=1.0\columnwidth]{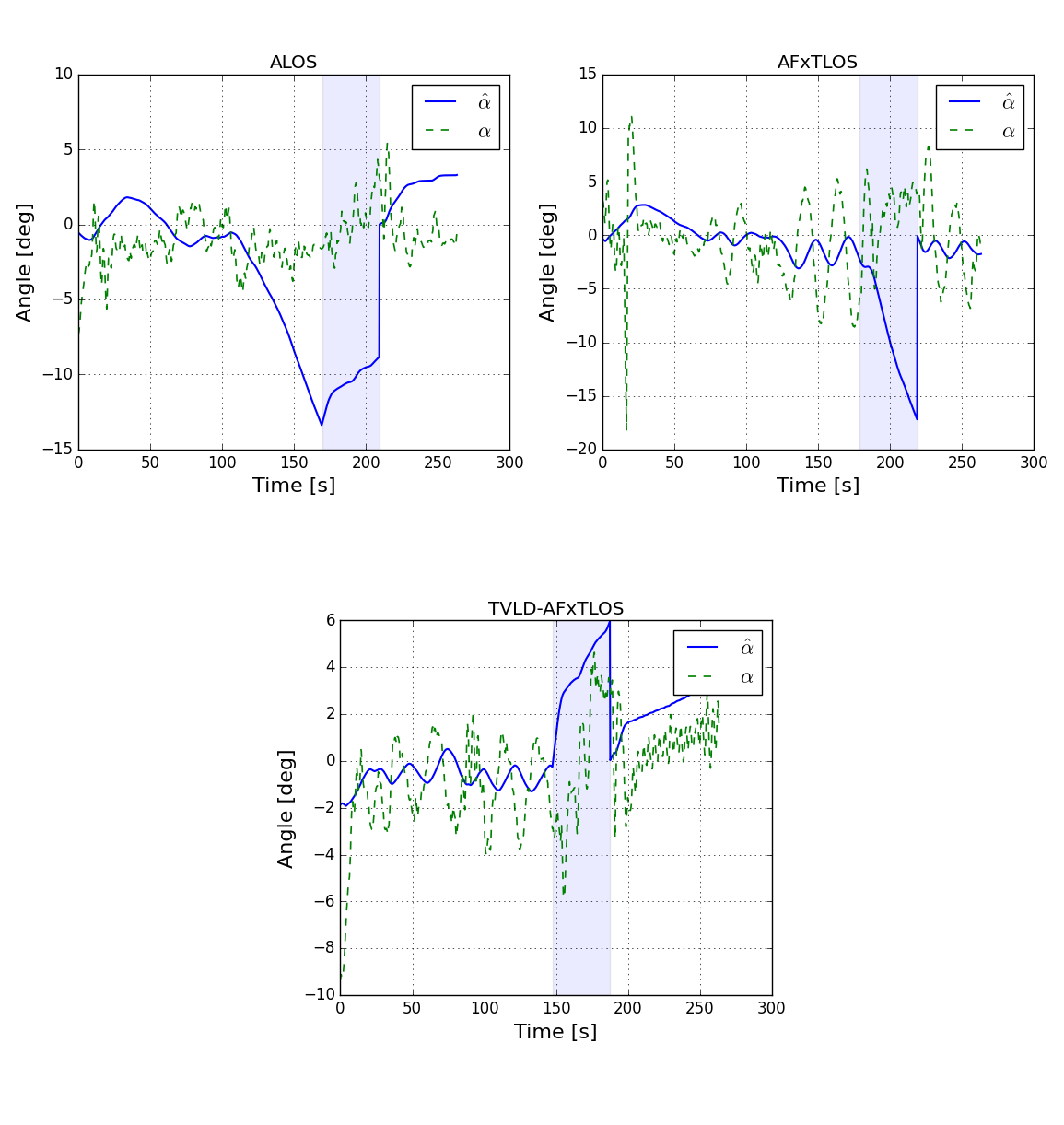}
        \caption{Vertical crab angle}
        \label{fig:crab_alpha_1030}
    \end{subfigure}
    \hfill
    \begin{subfigure}[t]{0.48\textwidth}
        \centering
        \includegraphics[width=1.0\columnwidth]{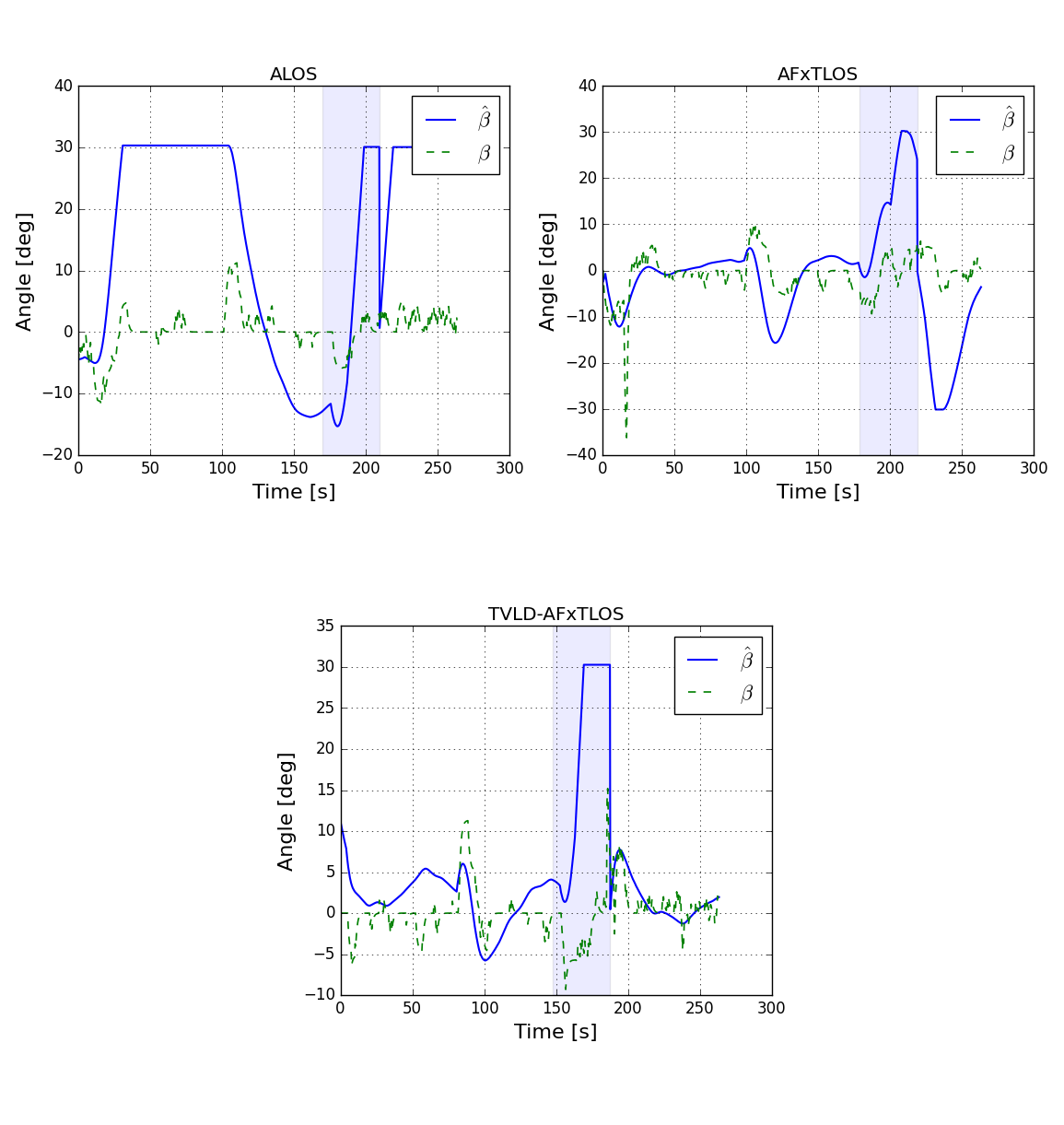}
        \caption{Horizontal crab angle}
        \label{fig:crab_beta_1030}
    \end{subfigure}
    \caption{(a): Comparison of vertical crab angles $\alpha$ (dashed green line) vs $\hat{\alpha}$ (blue line) for ALOS, AFxTLOS, and TVLD-AFxTLOS in Day 1. (b): Comparison of horizontal crab angles $\beta$ (dashed green line) vs $\hat{\beta}$ (blue line) for ALOS, AFxTLOS, and TVLD-AFxTLOS in Day 1.}
    \label{fig:crab_1030}
\end{figure*}

Abrupt variations in $\hat{\alpha}$ are observed during transitions between submerged and surface operation, highlighted by the shaded regions. Although the depth measurement itself is continuous due to the pressure sensor, these variations arise from an abrupt change in the reference depth originated by the switching between underwater and surface waypoints (blue shaded regions). Since the adaptation law depends on the vertical tracking error $Z_e$, the reference change induces corresponding transients in $\hat{\alpha}$. In general, AFxTLOS and TVLD-AFxTLOS closely follow the overall trend of $\alpha$, and, after the surfacing event, their estimates rapidly converge back to the measured value ($\alpha$). In contrast, ALOS struggles to capture the global trend and exhibits slower recovery following the surfacing event.

In the horizontal plane, the adaptive estimate $\hat{\beta}$ exhibits significantly larger variations. During submerged operation, $\hat{\beta}$ evolves smoothly for AFxTLOS and TVLD-AFxTLOS, following the overall trend. In contrast, ALOS quickly saturates to the maximum allowed values, which degrades performance, particularly during and after surfacing events. Large excursions and saturation of $\hat{\beta}$ are observed near surfacing events (blue shaded regions), which are caused by GPS corrections following dead-reckoning intervals. These corrections introduce abrupt changes in the horizontal tracking error $Y_e$, and since the adaptation law is driven by this error, they result in rapid variations in $\hat{\beta}$, bounded by the projection limits. Unlike ALOS, both AFxTLOS and TVLD-AFxTLOS are able to recover after surfacing. In particular, TVLD-AFxTLOS exhibits a smoother response with minimal overshoot compared to AFxTLOS. This behavior can be attributed to the time-varying look-ahead mechanism. As the cross-track error increases, the look-ahead distance $\Delta_h$ decreases, resulting in more aggressive guidance corrections. However, this reduction in $\Delta_h$ also decreases the effective gain of the adaptation law, as seen in the term $\Delta_h / \sqrt{\Delta_h^2 + k_1^2\sigmu{Y_e}^2}$. Consequently, the update of $\hat{\beta}$ is moderated during large error transients, such as those induced by GPS corrections, leading to smoother adaptive behavior and reduced overshoot.

Overall, the adaptive states $\hat{\alpha}$ and $\hat{\beta}$ capture the low frequency components associated with environmental disturbances and navigation drift. Among the evaluated strategies, TVLD-AFxTLOS exhibits smoother adaptive behavior, indicating improved robustness under hybrid underwater–surface operation.

\subsubsection{{Comparative experiment 2: Mild environmental disturbances}}
The second comparative experiment, conducted on Day 2, was performed under light air conditions, with an average wind speed of $1.5 \;\si{m/s}$ and gusts up to $4.0 \;\si{m/s}$. For this test, surfacing was forced to occur at the end of the mission to facilitate comparison. Additionally, the wind conditions generated minimal surface waves, which allowed for a more reliable GPS signal while the vehicle was on the surface. As a result, a more stable overall performance can be observed for all the controllers.

Similar to the previous run, Fig. \ref{fig:field_comparison_1113_horizontal} shows faster convergence of the AFxTLOS controller, closely followed by TVLD-AFxTLOS, which also exhibits steady performance but slightly slower convergence during the transitions between the first and second, and second and third mission legs. An oscillatory behavior and a slight steady-state cross-track error can be observed for the ALOS controller, which is not present in AFxTLOS and TVLD-AFxTLOS, {whose average tracking errors in the horizontal plane are $1.16\:\si{m}$ and $1.25\:\si{m}$ while ALOS has $1.32\:\si{m}$}. Consistent with the results of the previous run, TVLD-AFxTLOS produces smoother transitions; however, this smoothness results in slightly slower convergence. This behavior is mainly related to the configuration of $\Delta_{h_{min}}$, which matches the fixed $\Delta_h$ used in AFxTLOS and ALOS, effectively bounding its responsiveness.

Regarding the depth profile, AFxTLOS exhibits the steadiest overall descent toward the commanded $2\:\si{m}$ depth, although oscillatory behavior is visible during the second leg. In contrast, ALOS shows a steadier second leg performance with smaller oscillations, but with a slightly slower convergence to the commanded depth. Consistent with its horizontal-plane behavior, TVLD-AFxTLOS presents a smooth initial descent, followed by moderate oscillations during the mid-depth segment. These oscillations appear to be associated with variations in the look-ahead dynamics during the underwater leg, which slightly relax the convergence rate before stabilizing near the commanded depth.

\begin{figure*}[!ht]
  \centering
    \begin{subfigure}[t]{0.48\textwidth}
        \centering
        \includegraphics[width=1.0\columnwidth]{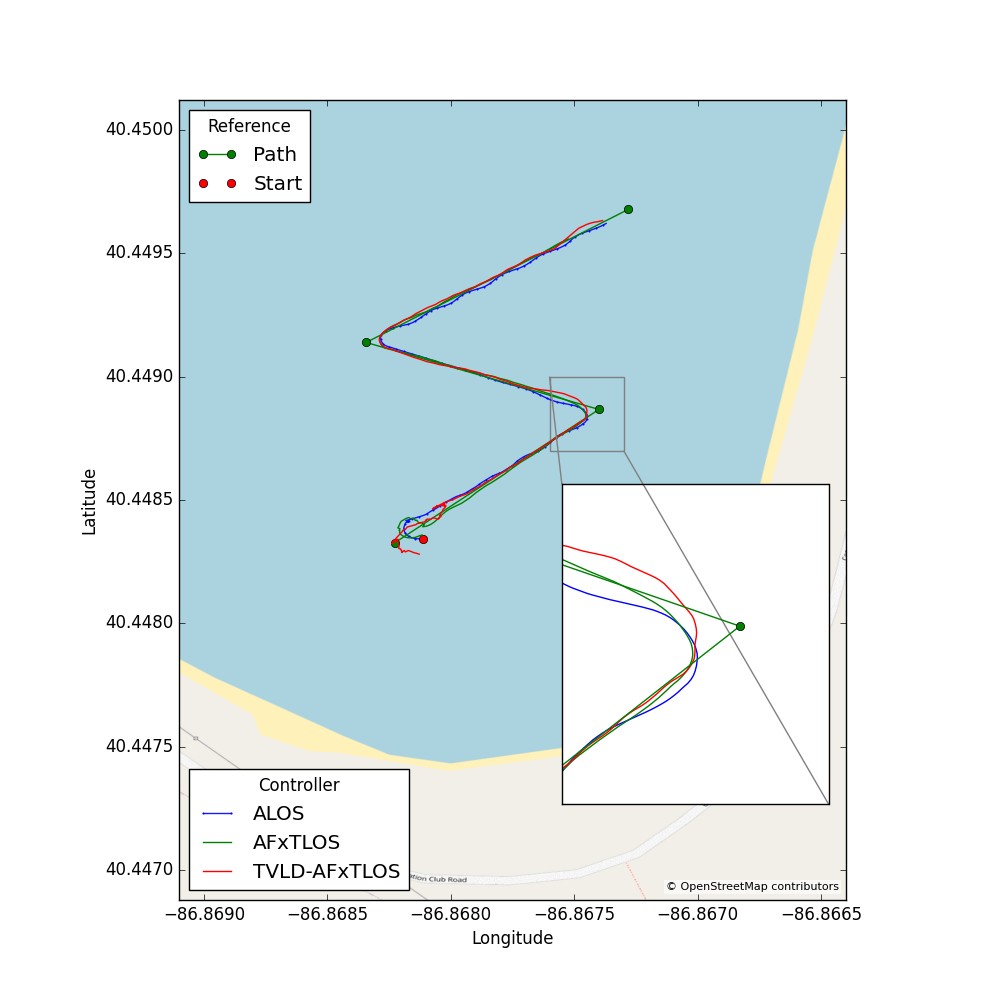}
        \caption{Horizontal trajectories}
        \label{fig:field_comparison_1113_horizontal}
    \end{subfigure}
    \hfill
    \begin{subfigure}[t]{0.48\textwidth}
        \centering
        \includegraphics[width=1.0\columnwidth]{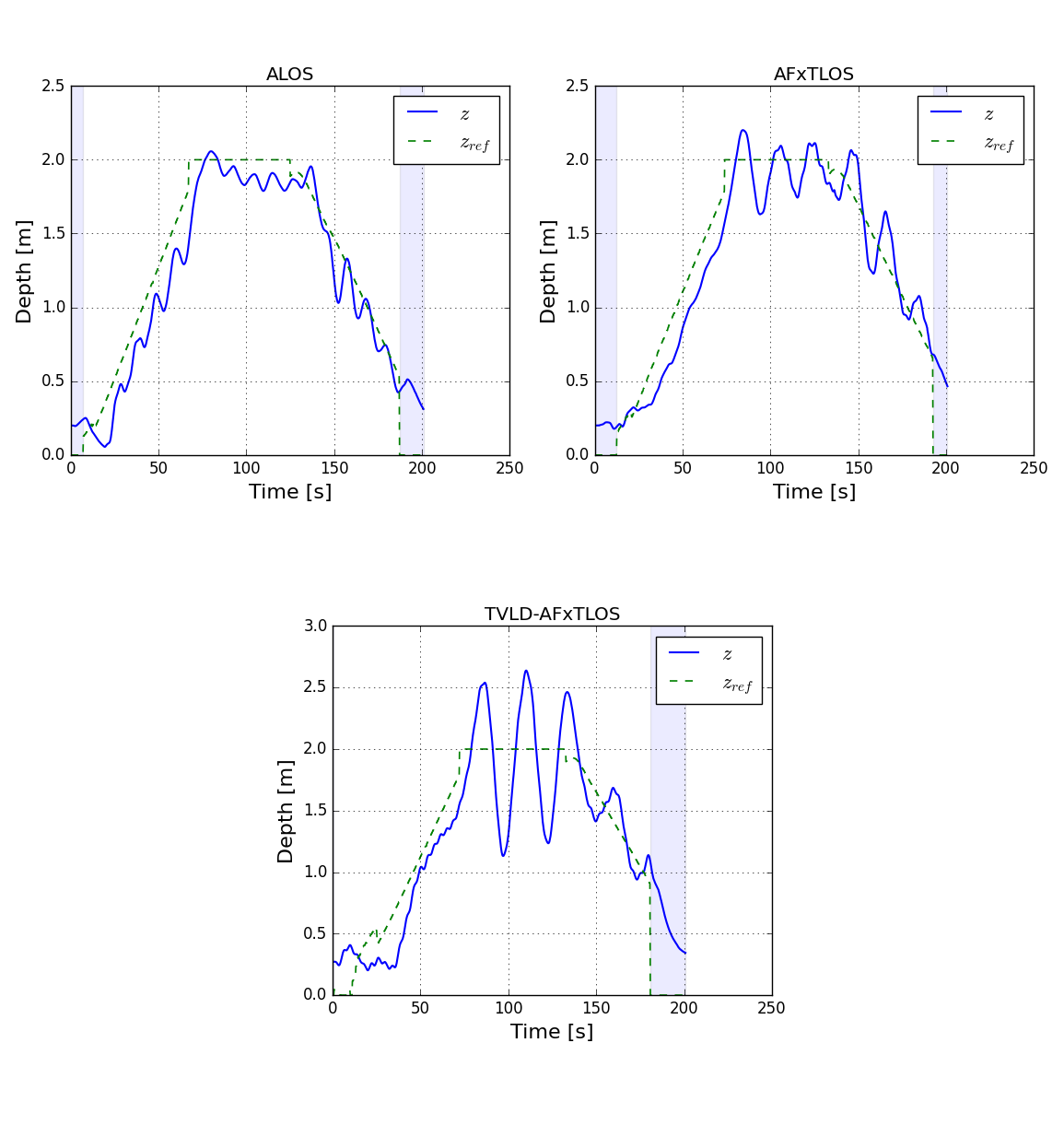}
        \caption{Depth profiles}
        \label{fig:field_comparison_1113_vertical}
    \end{subfigure}
    \caption{Field test under light air conditions. Left: Horizontal path-following comparison for ALOS (blue), AFxTLOS (green), and TVLD-AFxTLOS (red). AFxTLOS converges fastest, TVLD-AFxTLOS shows smoother but slightly slower transitions, and ALOS exhibits slower convergence with minor oscillations. Right: Depth profiles for the same mission. AFxTLOS achieves the fastest descent with increased oscillations, ALOS provides a steadier but slower response, and TVLD-AFxTLOS shows an intermediate.}
    \label{fig:field_comparison_1113}
\end{figure*}

The corresponding crab angle behavior is shown in Fig. \ref{fig:crab_1113}. Compared to Day 1, the adaptive estimates exhibit reduced magnitude and fewer abrupt variations, reflecting milder environmental disturbances and improved navigation consistency. In the vertical plane, $\hat{\alpha}$ remains bounded and smooth across all controllers, capturing the dominant trend of $\alpha$ while attenuating high-frequency fluctuations. Compared to AFxTLOS and TVLD-AFxTLOS, ALOS does not capture these variations and exhibits a slight bias with respect to the average value of the measured $\alpha$.

In the horizontal plane, $\hat{\beta}$ exhibits significantly reduced fluctuations compared to Day 1 and remains within a narrower range. This behavior indicates smaller variations in the cross-track error $Y_e$, likely due to reduced dead-reckoning drift and improved GPS corrections under calm surface conditions. Similar to the behavior observed on Day 1, ALOS saturates and subsequently fails to track the overall trend of $\beta$.

The TVLD-AFxTLOS strategy again produces smoother adaptive behavior compared to AFxTLOS, achieving a balance between smoothness and convergence to the measured value. This behavior can be attributed to the time-varying look-ahead mechanism, which moderates the influence of variations in $Y_e$ on the adaptation law, resulting in more stable estimates.

\begin{figure*}[!ht]
  \centering
    \begin{subfigure}[t]{0.48\textwidth}
        \centering
        \includegraphics[width=1.0\columnwidth]{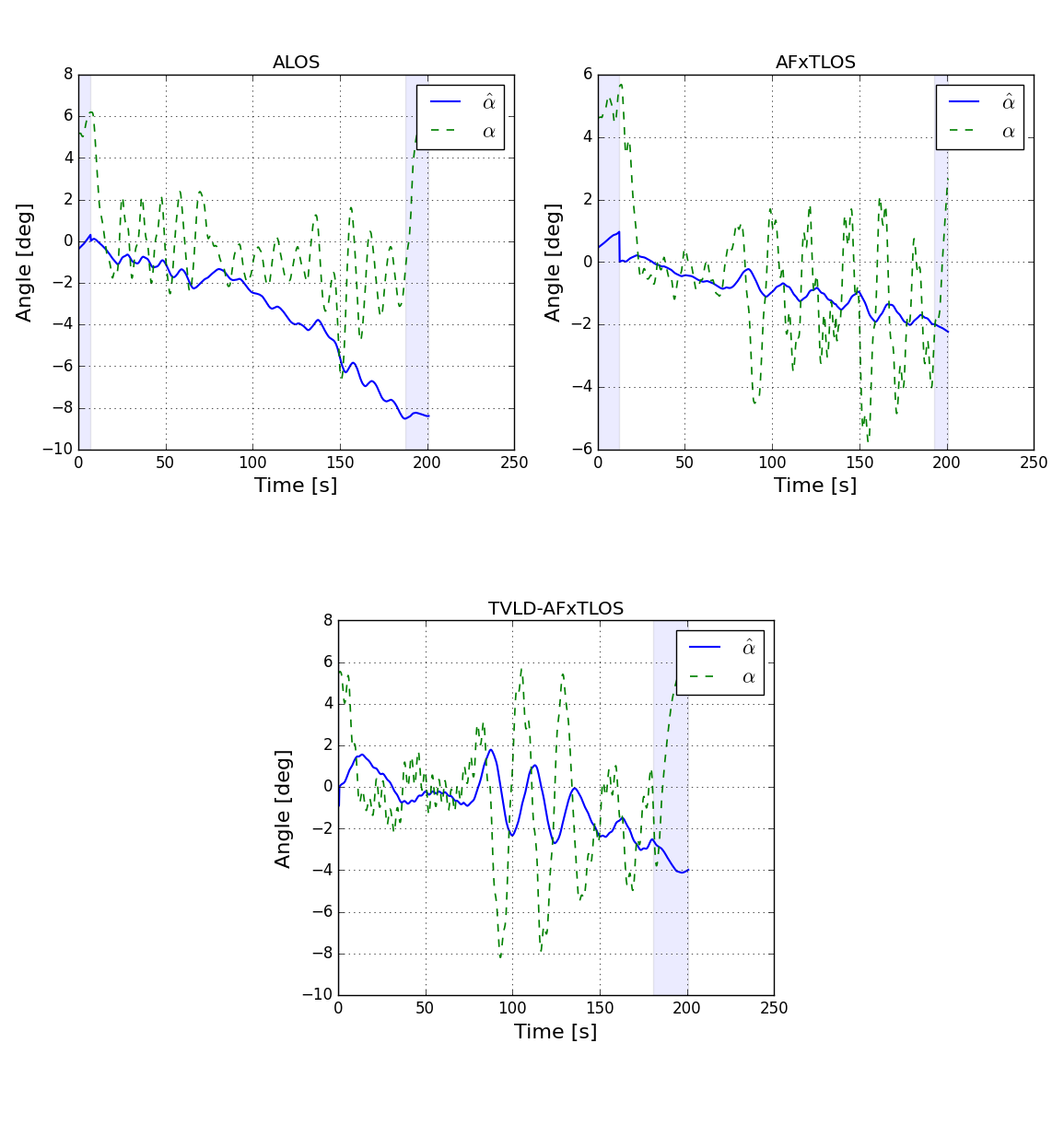}
        \caption{Vertical crab angle}
        \label{fig:crab_alpha_1113}
    \end{subfigure}
    \hfill
    \begin{subfigure}[t]{0.48\textwidth}
        \centering
        \includegraphics[width=1.0\columnwidth]{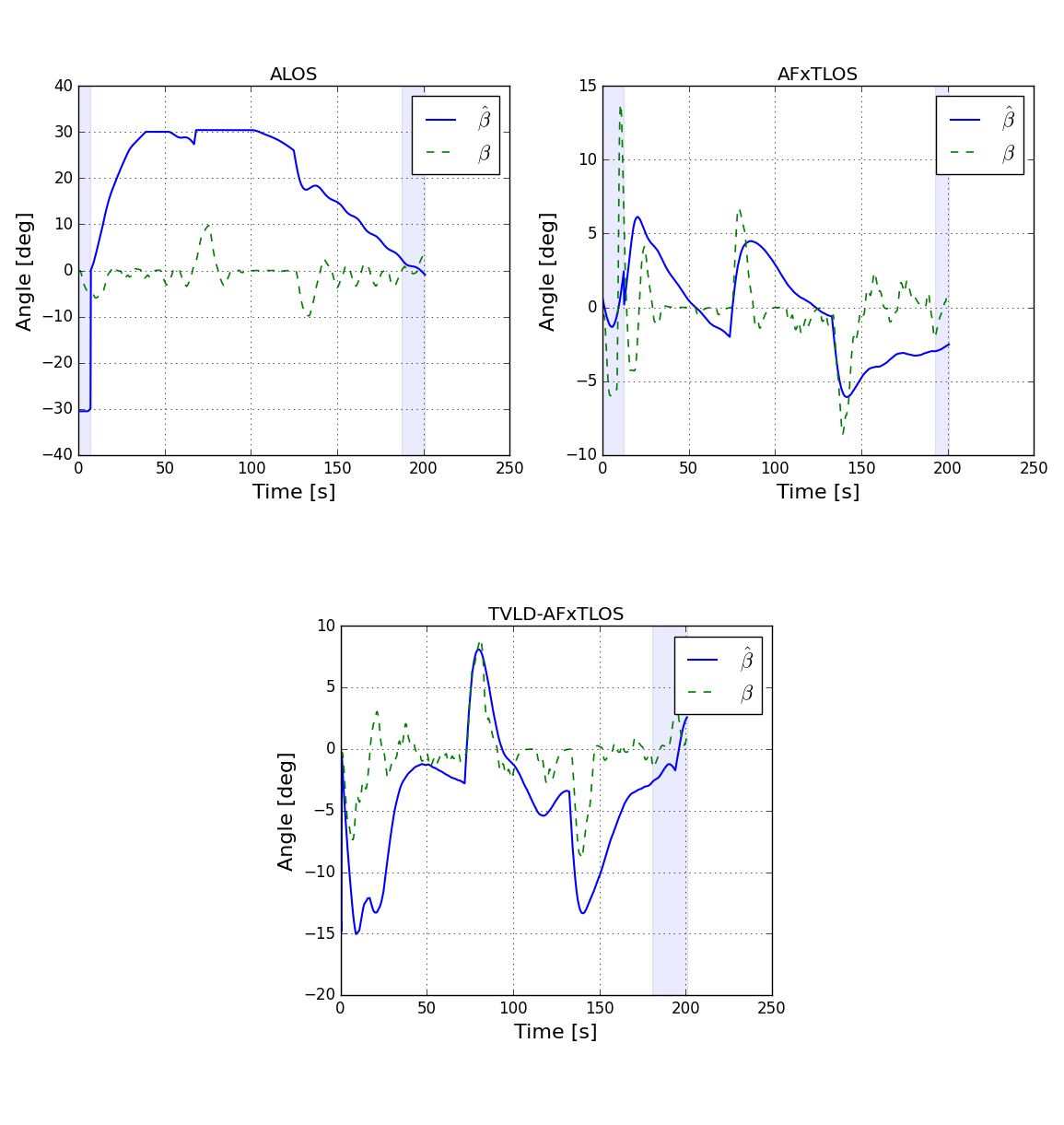}
        \caption{Horizontal crab angle}
        \label{fig:crab_beta_1113}
    \end{subfigure}
    \caption{(a): Comparison of vertical crab angles $\alpha$ (dashed green line) vs $\hat{\alpha}$ (blue line) for ALOS, AFxTLOS, and TVLD-AFxTLOS in Day 2. (b): Comparison of horizontal crab angles $\beta$ (dashed green line) vs $\hat{\beta}$ (blue line) for ALOS, AFxTLOS, and TVLD-AFxTLOS in Day 2.}
    \label{fig:crab_1113}
\end{figure*}

\subsubsection{Aggregate performance comparison}

To evaluate overall performance under different environmental conditions, statistical analysis was performed on all field operation data. Each controller performed two tasks, each task consisting of four trajectory runs, for a total of eight independent runs per controller. The comparative analysis covered horizontal and vertical tracking errors, and disturbance estimation performance.

Over two days of testing, the AFxTLOS-based method consistently reduced tracking error compared to ALOS. This improvement was particularly significant under moderate disturbances because the fixed-time convergence characteristic allowed the system to recover more quickly after encountering environmental disturbances or making navigation corrections. Overall, compared to ALOS, the proposed controller reduced the average tracking error by 56.35\% in straight-path evaluation and by 27.59\% in curved-path evaluation.

TVLD-AFxTLOS achieved the best balance between tracking accuracy and transient smoothness. While AFxTLOS achieved the fastest convergence, the additional variable look-ahead mechanism reduced overshoot during aggressive maneuvers and produced smoother disturbance adaptation.

\subsubsection{Discussion}

The field experiments highlight the practical scenarios where fixed-time guidance provides advantages over classical ALOS. For routine tasks with minimal interference and relatively relaxed accuracy requirements, ALOS remains attractive due to its simple structure and small control amplitude. However, for safety-critical operations requiring rapid recovery from interference (such as autonomous docking, close-range structural inspection, or navigation in confined environments), AFxTLOS can limit transient tracking deviations to predictable time intervals, thus providing higher reliability.

The proposed method introduces limited computational overhead because the fixed-time guidance law only relies on explicit error transformation and low-complexity adaptive updates. Unlike optimization-based guidance methods, this method does not require an online numerical solver, making real-time execution possible on the Iver3 backend computer. The algorithm operates at a frequency of $5 \si{Hz}$ without exhibiting a significant computational bottleneck.

The performance depends on appropriate tuning of convergence gains and look-ahead parameters. While increasing the fixed-time gain can enhance disturbance rejection and improve convergence speed, it can also exacerbate control actions and increase potential energy consumption; the TVLD-AFxTLOS mechanism provides additional flexibility by suppressing drastic responses when large transient errors occur.

Under strong environmental disturbances, the adaptive law can compensate for large sideslip changes within the designed operating range. However, the impact of degraded GPS signal quality on guidance performance must be considered, especially under the complex water surface conditions encountered during the test on Day 1. When the AUV resurfaces, unreliable GPS signals and delayed position corrections introduce abrupt, artificial step-changes in the computed tracking errors. Because the proposed fixed-time adaptation laws are designed for high responsiveness, they react sharply to these discrete localization jumps, resulting in parameter chattering and transient over-corrections in the control effort (as observed in the large excursions of $\hat{\beta}$). Feasible solutions to these problems include using an underwater acoustic positioning network \cite{UALN} to provide continuous underwater positioning feedback and applying Kalman filtering to obtain more reliable positioning estimates.Furthermore, a fault-tolerant mechanism can be added to the proposed LOS controller to reduce the influence of unreliable or delayed positioning signals and also practical actuator limitations such as saturation and hardware wear. This will be included in future works.

\section{Conclusion}\label{Conclusion}
{This paper proposes a robust AFxTLOS guidance strategy for 3D AUV path tracking. This method incorporates fixed-time stability theory, enabling rapid estimation and compensation of time-varying sideslip angles in strong current environments. Theoretical analysis and experimental results demonstrate that this method improves tracking performance, reducing the average tracking error by up to 69\% and the drift estimation error by 42\%. Field experiments with an Iver 3 AUV further validate the proposed approach, showing reductions in average tracking error of 56.35\% in straight-path evaluation and 27.59\% in curved-path evaluation compared with conventional ALOS guidance.}

{
Through field trials, this study also demonstrates that fixed-time guidance technology can be practically deployed on commercial underwater vehicles without modifying the existing underlying control architecture. This framework combines bounded-time convergence characteristics, an adaptive disturbance compensation mechanism, and field validation in real-world environments, thereby improving the reliability of autonomous underwater navigation in uncertain marine environments. These characteristics make this method particularly suitable for tasks such as long-duration inspections, environmental monitoring, and autonomous docking, where predictable transient performance is crucial.}

{
One drawback of this method is that the controller can cause more drastic attitude adjustments, potentially increasing energy consumption. Therefore, it is necessary to develop an energy-efficient adaptive guidance strategy based on existing research to balance tracking accuracy and control costs. This method will be tested in waters with stronger currents to further verify its robustness. Furthermore, future research will develop a multi-AUV cooperative control strategy based on AFxTLOS and introduce a dynamic obstacle avoidance mechanism to ensure the safety of AUV fleet operation.}


\appendices
\section{Proof of Lemma \ref{lemma: stability no adaptation}}\label{app: proof no adaptation}


    We first prove that the guidance laws lead to fixed-time stability if the crab angles are perfectly estimated, i.e. $\hat\alpha = \alpha$ and $\hat\beta = \beta$. In such a case, the error dynamics \eqref{eq: proof-error-1} and \eqref{eq: proof-error-2} can be simplified as
    \begin{align}
    \label{eq:proof-error-simple-1}
        \dot{Y}_e &= -U_h\frac{k_1\sigmu{Y_e}}{\sqrt{\Delta_h^2+k_1^2\sigmu{Y_e}^2}} \\ 
    \label{eq:proof-error-simple-2}
        \dot{Z}_e &= -U_v\frac{k_2\sigmu{Z_e}}{\sqrt{\Delta_v^2+k_2^2\sigmu{Z_e}^2}}+d(0,Y_e)
    \end{align}
    Then define a Lyapunov function candidate 
    \begin{equation}
        V_1=|Y_e|+|Z_e|
    \end{equation}
    whose derivative is
    \begin{equation}
    \begin{split}
        \dot{V}_1 &= \sign(Y_e)\dot{Y}_e + \sign(Z_e)\dot{Z}_e \\ 
        &= -\Omega_1k_1\sign(Y_e)\Bigl(|Y_e|^{1+1/\mu}\sign(Y_e)+|Y_e|\sign(Y_e)\\&\quad\; +|Y_e|^{1-1/\mu}\sign(Y_e)\Bigr)  -\Omega_2k_2\sign(Z_e)\Bigl(|Z_e|^{1+1/\mu}\cdot\\&\quad\;\cdot\sign(Z_e)+|Z_e|\sign(Z_e)+|Z_e|^{1-1/\mu}\sign(Z_e)\Bigr)  \\&\quad\;+ \sign(Z_e)d(0,Y_e) \\
        &= -\Omega_1k_1\Bigl(|Y_e|^{1+1/\mu}+|Y_e| +|Y_e|^{1-1/\mu}\Bigr) \\&\quad\;-\Omega_2k_2\Bigl(|Z_e|^{1+1/\mu}+|Z_e|+|Z_e|^{1-1/\mu}\Bigr)\\&\quad\;+\sign(Z_e) d(0,Y_e) 
    \end{split}
    \end{equation}
    where $\Omega_1={U_h}/{\sqrt{\Delta_h^2+k_1^2\sigmu{Y_e}^2}}$ and $\Omega_2={U_v}/{\sqrt{\Delta_v^2+k_2^2\sigmu{Z_e}}}$.     According to the definition of amplitude speeds and Assumptions  \ref{assum1} and \ref{assum3}, $U_h\geq0$ and $U_v>0$. And since the controller coefficients $k_1, k_2$ are positively definite, $\Omega_1\geq0$ and $\Omega_2>0$. Additionally, Assumptions \ref{assum1} and \ref{assum3} indicate that $U_v$ and $U_h$ is bounded, i.e. $U_{v, min}\leq U_v\leq U_{v,max}$ and $U_{h, min}\leq U_h\leq U_{h,max}$. Therefore, $\Omega_1$ and $\Omega_2$ are also bounded: $0<\Omega_1\leq U_{h,max}/\Delta_h$, $0<\Omega_2\leq U_{v,max}/\Delta_v$. Then, 
    \begin{equation}\begin{split}
        \dot{V}_1&\leq -\frac{k_1U_{h,min}}{\Delta_h}\Bigl(|Y_e|^{1+1/\mu} +|Y_e|^{1-1/\mu}\Bigr) \\&\quad\; -\frac{k_2U_{v,min}}{\Delta_v}\Bigl(|Z_e|^{1+1/\mu}+|Z_e|^{1-1/\mu}\Bigr) + |d(0,Y_e)| \\ 
        &\leq -\min\Bigl(\frac{k_1U_{h,min}}{\Delta_h}, \frac{k_2U_{v,min}}{\Delta_v} \Bigr)\Bigl( |Y_e|^{1+1/\mu} \\&\quad\; + |Z_e|^{1+1/\mu} + |Y_e|^{1-1/\mu} + |Z_e|^{1-1/\mu} \Bigr)+ |d(0,Y_e)| 
    \end{split}\end{equation}
    According to Jensen's inequality \cite{JensenInequlity}, $|Y_e|^{1+1/\mu}+|Z_e|^{1+1/\mu}\geq2^{1-(1+1/\mu)}(|Y_e|+|Z_e|)^{1+1/\mu}$ and $|Y_e|^{1-1/\mu}+|Z_e|^{1-1/\mu}
    \geq(|Y_e|+|Z_e|)^{1-1/\mu}$. Thus, 
    \begin{equation}\begin{split}
        \dot{V}_1&\leq-\min\Bigl(\frac{k_1U_{h,min}}{\Delta_h}, \frac{k_2U_{v,min}}{\Delta_v} \Bigr)\cdot\\&\quad\;\cdot\Bigl(2^{-1/\mu}V_1^{1+1/\mu}+V_1^{1-1/\mu}\Bigr) + |d(0,Y_e)| 
    \end{split}\end{equation}
    If the perturbation is zero, then by Lemma \ref{lemma2}, the system is fixed-time stable. Otherwise, since $0\leq|d(0,Y_e)|\leq U_{h,min}<\infty$ is finite, according to Lemma \ref{lemma3}, the origin of the system consisting of \eqref{eq:proof-error-simple-1} and \eqref{eq:proof-error-simple-2} is practical fixed-time stable. In addition, it can be seen that the subsystem \eqref{eq:proof-error-simple-1} is fixed-time stable according to Lemma \ref{lemma2} because $\sign(Y_e)\dot{Y}_e\leq -{k_1U_{h,max}}/{\Delta_h}\cdot\Bigl(|Y_e|^{1+1/\mu}+|Y_e|^{1-1/\mu}\Bigr)$. Hence, the subsystem \eqref{eq:proof-error-simple-2} can converge near the origin in fixed-time when $Y_e$ converges to zero.    This completes the proof of Lemma \ref{lemma: stability no adaptation}. 

\section{Proof of Lemma \ref{lemma: stability adaptation}}\label{app: proof adaptation}
To further prove the convergence of the proposed adaptation law, here we consider another Lyapunov function candidate 
\begin{equation}\begin{split}
    V_2 &= \int_0^{Y_e}{\sigmu{\tau}}d\tau+\frac{U_h}{\gamma_hk_1}(1-\cos(\tilde\beta)) \\&\quad\; + \int_0^{Z_e}{\sigmu{\tau}}d\tau+\frac{U_v}{\gamma_vk_2}(1-\cos(\tilde\alpha))
\end{split}\end{equation}
and its derivative yields
\begin{equation}
    \dot{V}_2 = \sigmu{Y_e}\dot{Y}_e+\frac{U_h}{\gamma_hk_1}\sin(\tilde\beta)\dot{\tilde\beta}+\sigmu{Z_e}\dot{Z}_e+\frac{U_v}{\gamma_vk_2}\sin(\tilde\alpha)\dot{\tilde\alpha}
\end{equation}
Based on Assumption \ref{assum2}, $\dot{\tilde\alpha}=-\dot{\hat\alpha}$ and $\dot{\tilde\beta}=-\dot{\hat\beta}$. Then, 
\begin{equation}\begin{split}
    \dot{V}_2 &=  - \sigmu{Y_e}U_h\Delta_h\Omega_4\Biggl[ \cos(\tilde\beta)\frac{k_1\sigmu{Y_e}}{\Delta_h}-\sin(\tilde\beta) \Biggr] \\&\quad\; - \frac{U_h}{k_1}\sin(\tilde\beta)\Delta_h\Omega_4\proj{(\hat\beta, k_1\sigmu{Y_e})}  \\&\quad\;- \sigmu{Z_e}U_v\Delta_v\Omega_3\biggl[ \cos(\Tilde{\alpha})\frac{k_2\sigmu{Z_e}}{\Delta_v}-\sin(\Tilde{\alpha}) \biggr]\\&\quad\; -\frac{U_v}{k_2}\sin(\tilde\alpha)\Delta_v\Omega_3\proj{(\hat\alpha, k_2\sigmu{Z_e})}  \\
    &= - U_h\Omega_4\cos(\tilde\beta)k_1\sigmu{Y_e}^2+U_h\Delta_h\Omega_4\sin(\tilde\beta)\sigmu{Y_e} \\&\quad\;-U_v\Omega_3\cos(\tilde\alpha)k_2\sigmu{Z_e}^2+U_v\Delta_v\Omega_3\sin(\tilde\alpha)\sigmu{Z_e} \\&\quad\; - U_h\sin(\tilde\beta)\Delta_h\Omega_4\sigmu{Y_e}\cdot{Q_Y} \\&\quad\; - U_v\sin(\tilde\alpha)\Delta_v\Omega_3\sigmu{Z_e}\cdot{Q_Z}
\end{split}\end{equation}
 where $\Omega_3 = {1}/{\sqrt{\Delta_v^2+k_2^2(\lceil{Z_e,\mu}\rfloor)^2}}$, \\ $\Omega_4 = {1}/{\sqrt{\Delta_h^2+k_1^2(\lceil{Y_e,\mu}\rfloor)^2}}$,\\ $Q_Y = \begin{cases}
     1-c(\hat\beta), & if\:|\hat\beta|>M_\beta\:\&\:\hat\beta k_1\sigmu{Y_e}>0 \\
     1, & else
 \end{cases}$ \\ and $Q_Z = \begin{cases}
     1-c(\hat\alpha), & if\:|\hat\alpha|>M_\alpha\:\&\:\hat\alpha k_2\sigmu{Z_e}>0 \\
     1, & else
 \end{cases}$. 
 According to the parameter projection function \eqref{eq: AFxTLOS-projection}, the definition of $c(\cdot)$ indicates that $0<c(\hat\beta)\leq1$. If $\hat\beta>M_\beta$, then $\beta-\hat\beta=\tilde\beta<0$ and thus $\sin(\tilde\beta)<0$; since in this case it should follows $\sigmu{Y_e}>0$ to have $\hat\beta^T\sigmu{Y_e}>0$, it can be concluded that $c(\hat\beta)\sin(\tilde\beta)\sigmu{Y_e}<0$. 
 If $\hat\beta<-M_\beta$, then $\beta-\hat\beta=\tilde\beta>0$ and thus $\sin(\tilde\beta)>0$; since in this case it should follows $\sigmu{Y_e}<0$ to have $\hat\beta^T\sigmu{Y_e}>0$, it can be concluded that that $c(\hat\beta)\sin(\tilde\beta)\sigmu{Y_e}<0$. Therefore, it is proved to guarantee that $c(\hat\beta)\sin(\tilde\beta)\sigmu{Y_e}<0$ and $c(\hat\alpha)\sin(\tilde\alpha)\sigmu{Z_e}<0$. As a result, 
 \begin{equation}\begin{split}
     \dot{V}_2&\leq -U_h\Omega_4\cos(\tilde\beta)k_1\sigmu{Y_e}^2-U_v\Omega_3\cos(\tilde\alpha)k_2\sigmu{Z_e}^2\\&\quad\;+c(\hat\beta)U_h\Delta_h\Omega_4\sin(\tilde\beta)\sigmu{Y_e}\\&\quad\;+c(\hat\alpha)U_v\Delta_v\Omega_3\sin(\tilde\alpha)\sigmu{Z_e}\\
     &\leq -U_h\Omega_4\cos(\tilde\beta)k_1\sigmu{Y_e}^2-U_v\Omega_3\cos(\tilde\alpha)k_2\sigmu{Z_e}^2
 \end{split}\end{equation}
Therefore, the system is stable, and the convergence of tracking and estimation errors is guaranteed in the presence of zero perturbation.  This completes the proof of Lemma \ref{lemma: stability adaptation}.

\section{Proof of Theorem \ref{theorem: convergence}}
\label{proof: convergence}

According to the definitions of $M_\vartheta$ and $M_{\hat{\vartheta}}$, the estimates and true values of the crab angles are finitely bounded, so based on Lemma \ref{lemma2}, the system consisting of the control law and adaption law is practical fixed-time stable and the tracking error converges to the residual set as defined in \eqref{eq: fixed residual set} before the crab angle estimations are fully converged. 
As a result, based on Lemmas \ref{lemma: stability no adaptation} and \ref{lemma: stability adaptation}, with the finite estimation and tracking errors, the overall path following adaptive control system can converge within a finite residual set around the origin in fixed-time. 
This completes the proof of Theorem \ref{theorem: convergence}.

\bibliographystyle{IEEEtran}
\bibliography{ifacconf}

\vfill\pagebreak

\end{document}